\documentclass[twoside,11pt,hidelinks]{article}

\usepackage{jmlr2e}
\usepackage[utf8]{inputenc} 
\usepackage{hyperref}       
\usepackage[T1]{fontenc}    
\usepackage{url}            
\usepackage{booktabs}       
\usepackage{amsfonts}       
\usepackage{nicefrac}       
\usepackage{microtype}      
\usepackage{enumitem}
\usepackage{algorithm,algorithmic}
\newcommand{\eq}[1]{{Eq~(#1)}}

\usepackage{pifont}
\newcommand{\xmark}{\ding{55}}
\newcommand{\lenfifteen}{13cm}
\usepackage{amsbsy}
\usepackage{amsmath}
\usepackage{graphicx}
\usepackage{subfigure}
\usepackage{color}
\usepackage{booktabs}

\usepackage{natbib}
\usepackage{etoolbox}
\makeatletter
\patchcmd{\@makefntext}{\insertfootnotetext{#1}}{\insertfootnotetext{\scriptsize#1}}{}{}
\makeatother

\usepackage[normalem]{ulem} 

\newcommand{\vE}{{\mathbf{E}}}

\newcommand{\vP}{{\mathbf{P}}}

\newcommand{\cA}{{\mathcal{A}}}

\newcommand{\cD}{{\mathcal{D}}}

\newcommand{\cH}{{\mathcal{H}}}

\newcommand{\cL}{{\mathcal{L}}}
\newcommand{\cM}{{\mathcal{M}}}

\newcommand{\cO}{{\mathcal{O}}}

\newcommand{\cS}{{\mathcal{S}}}
\newcommand{\cT}{{\mathcal{T}}}

\newcommand{\grad}{{\nabla}}    

\DeclareMathOperator*{\argmin}{arg\,min}
\DeclareMathOperator*{\argmax}{arg\,max}

\newcommand{\bc}{\begin{center}}
\newcommand{\ec}{\end{center}}

\newcommand{\bdm}{\begin{displaymath}}
\newcommand{\edm}{\end{displaymath}}

\newcommand{\beq}{\begin{equation}}
\newcommand{\eeq}{\end{equation}}

\newcommand{\bfl}{\begin{flushleft}}
\newcommand{\efl}{\end{flushleft}}

\newcommand{\bt}{\begin{tabbing}}
\newcommand{\et}{\end{tabbing}}

\newcommand{\beqn}{\begin{eqnarray}}
\newcommand{\eeqn}{\end{eqnarray}}

\newcommand{\beqs}{\begin{align*}} 
\newcommand{\eeqs}{\end{align*}}  

\newtheorem{condition}{Condition}
\newtheorem{assumption}{Assumption}

\usepackage{pdfpages}
\usepackage{mathtools}
\usepackage{amsmath}
\usepackage{wrapfig}
\allowdisplaybreaks
\usepackage{amssymb}
\DeclarePairedDelimiter{\norm}{\lVert}{\rVert}
\usepackage{balance}
\usepackage[titletoc]{appendix}

\ShortHeadings{Ranking Policy Gradient}{Lin and Zhou}
\firstpageno{1}

\begin{document}

\title{Ranking Policy Gradient}

\author{\name Kaixiang \ Lin \email linkaixi@msu.edu \\
       \addr Department of Computer Science and Engineering\\
		Michigan State University \\
		East Lansing, MI 48824-4403, USA
       \AND
       \name Jiayu \ Zhou \email jiayuz@msu.edu \\
       \addr Department of Computer Science and Engineering\\
		Michigan State University \\
		East Lansing, MI 48824-4403, USA}


\maketitle

\begin{abstract} 

\noindent Sample inefficiency is a long-lasting problem in reinforcement learning (RL). 
The state-of-the-art estimates the optimal action values while it usually
involves an extensive search over the state-action space and unstable optimization.
Towards the sample-efficient RL, we propose
\textit{ranking policy gradient} (RPG), a policy gradient method that learns
the optimal rank of a set of discrete actions.  To accelerate the learning
of policy gradient methods, we establish the equivalence between maximizing the lower bound of return and imitating a near-optimal policy without accessing any oracles. 
These results lead to a general off-policy learning framework, 
which preserves the optimality, reduces variance, and improves the sample-efficiency.  
Furthermore, the sample complexity of RPG does not depend on the dimension of state space,
which enables RPG for large-scale problems.
We conduct extensive experiments showing that when consolidating with the off-policy learning
framework, RPG substantially reduces the sample complexity, comparing to the
state-of-the-art.


\end{abstract}

\section{Introduction}

One of the major challenges in reinforcement learning (RL) is the high \emph{sample
complexity}~\citep{kakade2003sample}, which is the number of samples must be
collected to conduct successful learning. 
There are different reasons leading to poor sample efficiency of
RL~\citep{yu2018towards}.  Because policy gradient algorithms directly
optimizing return estimated from rollouts (e.g.,
\textsc{Reinforce}~\citep{williams1992simple}) could suffer from high
variance~\citep{sutton2018reinforcement}, value function baselines were
introduced by actor-critic methods to reduce the variance and improve the
sample-efficiency. However, since a value function is associated with a
certain policy,  the samples collected by former policies cannot be readily
used without complicated manipulations~\citep{degris2012off} and extensive
parameter tuning~\citep{nachum2017bridging}. Such an on-policy requirement
increases the difficulty of sample-efficient learning. 

On the other hand, off-policy methods, such as one-step $Q$-learning~\citep{watkins1992q} and
variants of deep $Q$ networks
(DQN)~\citep{mnih2015human,hessel2017rainbow,dabney2018implicit,van2016deep,schaul2015prioritized}, enjoys the
advantage of learning from any trajectory sampled from the same environment
(i.e., off-policy learning), are currently among the most sample-efficient
algorithms.  These algorithms, however, often require extensive
searching~\citep[Chap. 5]{bertsekas1996neuro} over the large state-action space
to estimate the optimal action value function. Another deficiency is that, the combination of 
off-policy learning, bootstrapping, and function approximation, making up
what~\cite{sutton2018reinforcement} called the "deadly triad", can easily lead to unstable
or even divergent learning~\citep[Chap. 11]{sutton2018reinforcement}. 
These inherent issues limit their sample-efficiency. 

Towards addressing the aforementioned challenge, we approach the
sample-efficient reinforcement learning from a ranking perspective. Instead of
estimating optimal action value function, we concentrate on learning optimal
rank of actions. The rank of actions depends on the \textit{relative action
values}. As long as the relative action values preserve the same rank of
actions as the optimal action values ($Q$-values), we 
choose the same optimal action. 
To learn optimal relative action values, we propose the \emph{ranking
policy gradient (RPG)} that optimizes the actions' rank with respect to the
long-term reward by learning the pairwise relationship among actions. 

Ranking Policy Gradient (RPG) that directly optimizes relative action values
to maximize the return is a policy gradient method. The track of off-policy
actor-critic methods~\citep{degris2012off,gu2016q,wang2016sample} have made
substantial progress on improving the sample-efficiency of policy gradient. 
However, the fundamental difficulty of learning stability associated with the
bias-variance trade-off remains~\citep{nachum2017bridging}. In this work, we first exploit the
equivalence between RL optimizing the lower bound of return and supervised
learning that imitates a specific optimal policy. Build upon this theoretical
foundation, we propose a general off-policy learning framework that equips
the generalized policy iteration~\citep[Chap. 4]{sutton2018reinforcement} with
an external step of supervised learning. The proposed off-policy
learning not only enjoys the property of optimality preserving (unbiasedness),
but also largely reduces the variance of policy gradient because of its
independence of the horizon and reward scale. Furthermore, this learning paradigm leads to a  sample complexity analysis of large-scale MDP, 
in a non-tabular setting without the linear dependence on the state space. 
Based on our sample-complexity analysis, we define the exploration efficiency that quantitatively
evaluates different exploration methods. 
Besides, we empirically show that there is a trade-off between optimality and 
sample-efficiency, which is well aligned with our theoretical indication. 
Last but not least, we demonstrate that the proposed approach, consolidating the RPG
with off-policy learning, significantly outperforms the
state-of-the-art~\citep{hessel2017rainbow,
bellemare2017distributional,dabney2018implicit,mnih2015human}.

\vspace{-1em}
\section{Related works}

\noindent\textbf{Sample Efficiency.}
The sample efficient reinforcement learning can be roughly divided into two categories.
The first category includes variants of $Q$-learning~\citep{mnih2015human,schaul2015prioritized,van2016deep,hessel2017rainbow}. The main advantage of $Q$-learning methods
is the use of off-policy learning, which is essential towards sample efficiency.
The representative DQN~\citep{mnih2015human} introduced deep neural network in $Q$-learning,
which further inspried a track of successful DQN variants such as 
Double DQN~\citep{van2016deep}, Dueling networks~\citep{wang2015dueling}, 
prioritized experience replay~\citep{schaul2015prioritized}, 
and \textsc{Rainbow}~\citep{hessel2017rainbow}.
The second category is the actor-critic approaches. Most of recent works~\citep{degris2012off,wang2016sample,gruslys2018reactor}
in this category leveraged importance sampling by re-weighting the samples 
to correct the estimation bias and reduce variance. 
The main advantage is in the wall-clock times due 
to the distributed framework, firstly presented in~\citep{mnih2016asynchronous}, instead
of the sample-efficiency. As of the time of writing, 
the variants of DQN~\citep{hessel2017rainbow,dabney2018implicit,bellemare2017distributional,schaul2015prioritized,van2016deep} are among the algorithms of most sample efficiency, 
which are adopted as our baselines for comparison.


\noindent\textbf{RL as Supervised Learning.} 
Many efforts have focused on developing the connections between RL and supervised learning,
such as Expectation-Maximization algorithms~\citep{dayan1997using,peters2007reinforcement,kober2009policy,abdolmaleki2018maximum}, Entropy-Regularized RL~\citep{oh2018self,haarnoja2018soft}, 
and Interactive Imitation Learning (IIL)~\citep{daume2009search,syed2010reduction,ross2010efficient,ross2011reduction,sun2017deeply,hester2018deep,osa2018algorithmic}. EM-based approaches 
apply the probabilistic framework to formulate the RL problem maximizing a lower bound of the return 
as a re-weighted regression problem, while it requires on-policy estimation on 
the expectation step. Entropy-Regularized RL optimizing entropy augmented objectives
can lead to off-policy learning without the usage of importance sampling while it converges
to soft optimality~\citep{haarnoja2018soft}.

Of the three tracks in prior works, the IIL is most closely 
related to our work. The IIL works firstly pointed out the connection between imitation 
learning and reinforcement learning~\citep{ross2010efficient,syed2010reduction,ross2011reduction} 
and explore the idea of facilitating reinforcement learning by imitating experts. 
However, most of imitation learning algorithms assume the access to the expert policy
or demonstrations. The off-policy learning framework proposed in this paper can be interpreted as an 
online imitation learning approach that constructs expert demonstrations during the
exploration without soliciting experts, 
and conducts supervised learning to maximize return at the same time.
In short, our approach is different from prior arts 
in terms of at least one of the following aspects:
objectives, oracle assumptions, the optimality of learned policy, and on-policy requirement. 
More concretely, the proposed method is able to learn optimal policy in terms of long-term reward,
without access to the oracle (such as expert policy or expert demonstration) and it can be trained both
empirically and theoretically in an off-policy fashion. 
A more detailed discussion of the related work on reducing RL to supervised learning is provided in Appendix~\ref{sec:priworkRL2SL}.


\noindent\textbf{PAC Analysis of RL.} 
Most existing studies on sample complexity analysis~\citep{kakade2003sample,strehl2006pac,kearns2000approximate,strehl2009reinforcement,
krishnamurthy2016pac,jiang2017contextual,jiang2018open,zanette2019tighter} 
are established on the value function estimation. 
The proposed approach leverages the probably approximately correct framework~\citep{valiant1984theory}
in a different way such that it does not rely on the value function. Such independence directly leads
to a practically sample-efficient algorithm for large-scale MDP, 
as we demonstrated in the experiments.


\section{Notations and Problem Setting} 
\label{sec:notations_and_problem_setting}

In this paper, we consider a finite horizon $T$, discrete time Markov Decision Process (MDP) 
with a finite discrete state space $\cS$ and for each state $s \in \cS$, 
the action space $\cA_s$ is finite. The environment dynamics is denoted as 
$\vP = \{p(s'|s,a), \forall s, s'\in \cS, a\in \cA_s\}$. We note that the dimension of action space
can vary given different states. We use $m= \max_{s}\|\cA_s\|$ to denote the maximal action 
dimension among all possible states. 
Our goal is to maximize the expected sum of positive rewards, or
return $J(\theta) = \vE_{\tau,\pi_\theta} [\sum_{t=1}^Tr(s_t, a_t)]$, where $0<r(s, a)<\infty,  \forall s, a$. In this case, the optimal
deterministic Markovian policy always exists~\citep{puterman2014markov}[Proposition 4.4.3].
The upper bound of trajectory reward ($r(\tau)$) is denoted as $R_{\max} = \max_{\tau} r(\tau)$. 
A comprehensive list of notations is elaborated in Table~\ref{tb:summarynotation}. 

\begin{table}[t!]
{\renewcommand\arraystretch{1.25}
\begin{tabular}{|l|l|l|} \hline
Notations & \multicolumn{2}{l|}{Definition} \\ \hline\hline
$\lambda_{ij}$ & \multicolumn{2}{p{\lenfifteen}|}{\raggedright The discrepancy of the relative action value of action $i$ and action $j$. $\lambda_{ij}= \lambda_i - \lambda_j$, where $\lambda_i = \lambda(s, a_i).$ Notice that the value here is not the estimation of
return, it represents which action will have relatively higher return if followed.} \\ \hline
$Q^{\pi}(s, a)$ & \multicolumn{2}{p{\lenfifteen}|}{\raggedright 
The action value function or equivalently the estimation of
return taking action $a$ at state $s$, following policy $\pi$.} \\ \hline
$p_{ij}$  & \multicolumn{2}{p{\lenfifteen}|}{\raggedright $p_{ij} = P(\lambda_i > \lambda_j)$ denotes the probability that $i$-th action is to be ranked higher than $j$-th action. Notice that $p_{ij}$ is controlled by
$\theta$ through $\lambda_i, \lambda_j$} \\ \hline
$\tau$  & \multicolumn{2}{p{\lenfifteen}|}{\raggedright A trajectory $\tau = \{s(\tau, t), a(\tau,t)\}_{t=1}^T$ collected 
from the environment. It is worth noting that this trajectory is not associated with any policy. It only represents a series of state-action pairs. We also use the abbreviation $s_t = s(\tau, t)$, $a_t = a(\tau, t)$.} \\ \hline
$r(\tau)$  & \multicolumn{2}{p{\lenfifteen}|}{\raggedright The trajectory reward 
$r(\tau)= \sum_{t=1}^Tr(s_t, a_t)$ is the sum of reward along one trajectory. 
} \\ \hline
$R_{\max}$ & \multicolumn{2}{p{\lenfifteen}|}{\raggedright $R_{\max}$ is the maximal possible trajectory reward, i.e., $R_{\max} = \max_{\tau} r(\tau)$. Since we focus on MDPs with finite horizon and immediate reward, therefore the 
trajectory reward is bounded. 
} \\ \hline
$\sum_\tau$  & \multicolumn{2}{p{\lenfifteen}|}{\raggedright 
The summation over all possible trajectories $\tau$. } \\ \hline
$p(\tau)$  & \multicolumn{2}{p{\lenfifteen}|}{\raggedright 
The probability of a specific trajectory is 
collected from the environment given policy $\pi_{\theta}$. $p_{\theta}(\tau) =p(s_0)\Pi_{t=1}^{T}\pi_{\theta}(a_t|s_t)p(s_{t+1}|s_t, a_t)$ 
} \\ \hline
$\cT$  & \multicolumn{2}{p{\lenfifteen}|}{\raggedright 
The set of all possible near-optimal trajectories. $|\cT|$ denotes the number of near-optimal trajectories in $\cT$. 
} \\ \hline
$n$  & \multicolumn{2}{p{\lenfifteen}|}{\raggedright 
The number of training samples or equivalently state action pairs sampled from 
uniformly (near)-optimal policy.} \\ \hline
$m$  & \multicolumn{2}{p{\lenfifteen}|}{\raggedright 
The number of discrete actions.} \\ \hline
\end{tabular}}
\caption{Notations}
\label{tb:summarynotation}
\end{table}


\section{Ranking Policy Gradient}
\label{sec:rpg}
Value function estimation is widely used in advanced RL 
algorithms~\citep{mnih2015human,mnih2016asynchronous,schulman2017proximal,gruslys2018reactor,hessel2017rainbow,dabney2018implicit} 
to facilitate the learning process. 
In practice, the on-policy requirement of value function estimations in
actor-critic methods has largely increased the difficulty of sample-efficient
learning~\citep{degris2012off,gruslys2018reactor}. 
With the advantage of off-policy learning, 
the DQN~\citep{mnih2015human} variants are currently among the
most sample-efficient algorithms~\citep{hessel2017rainbow,dabney2018implicit,bellemare2017distributional}. 
For complicated tasks, the value function can align with
the relative relationship of action's return, but the absolute
values are hardly accurate~\citep{mnih2015human,ilyas2018deep}. 

The above observations motivate us to look at the decision phase of
RL from a different prospect: Given a state, the decision
making is to perform a \emph{relative comparison} over available actions and then 
choose the best action, which can lead to relatively higher return than
others. Therefore, an alternative solution is to learn the optimal rank of the actions,
instead of deriving policy from the action values. 
In this section, we show how to optimize the rank of actions to maximize the return, 
and thus avoid the necessity of accurate estimation for
optimal action value function. To learn the rank of actions, we focus on learning 
\textit{relative action value} ($\lambda$-values), defined as follows:
\begin{definition}[Relative action value ($\lambda$-values)]
	For a state $s$, the relative action values of $m$ actions ($\lambda(s, a_k), k=1,...,m$) is a list of scores that denotes the rank of actions. If $\lambda(s, a_i) > \lambda(s, a_j)$, then action $a_i$ is
	ranked higher than action $a_j$.
	\label{def:rav}
\end{definition}
The optimal relative action values should preserve the same optimal action as the optimal action values:
$$\argmax_{a} \lambda(s, a) = \argmax_{a} Q^{\pi_*}(s, a) $$
where $Q^{\pi_*}(s, a_i)$ and $\lambda(s, a_i)$ represent the optimal action value and the relative action value of action $a_i$, respectively. 
We omit the model parameter $\theta$ in $\lambda_{\theta}(s, a_i)$ for concise presentation. 
\begin{remark}
The $\lambda$-values are different from the advantage function $A^{\pi}(s,a) = Q^{\pi}(s,a) - V^{\pi}(s)$. The advantage functions quantitatively show 
the difference of return taking different actions following the current 
policy $\pi$. The $\lambda$-values only determine the relative order of actions and its magnitudes are not the estimations of returns. 
\end{remark}


To learn the $\lambda$-values, we can construct a probabilistic model of $\lambda$-values such that the best action has the highest probability to be selected than others.
Inspired by learning to rank~\citep{burges2005learning},
we consider the pairwise relationship among all actions, by modeling
the probability (denoted as $p_{ij}$) of an action $a_i$ to be ranked 
higher than any action $a_j$ as follows:
\begin{align}
p_{ij} = \frac{\exp(\lambda(s, a_i) - \lambda(s, a_j))}{1+ \exp(\lambda(s, a_i) - \lambda(s, a_j))},
\label{eq:rpgt:pij}
\end{align}
where $p_{ij} = 0.5$ means the relative action value of $a_i$ is same as that of the action $a_j$, 
$p_{ij}>0.5$ indicates that the action $a_i$ is ranked higher than $a_j$. 
Given the independent Assumption~\ref{ass:rpgind}, we can represent the probability 
of selecting one action as the multiplication of a set of pairwise probabilities in~\eq{\ref{eq:rpgt:pij}}. Formally, we define the pairwise ranking policy in \eq{\ref{eq:rpg:policy}}.  
Please refer to Section~\ref{sec:asm1} in the Appendix for the discussions on feasibility  
of Assumption~\ref{ass:rpgind}.

\begin{definition}
The pairwise ranking policy is defined as:
\begin{align}
	\pi(a=a_i|s) = \Pi_{j=1, j\neq i}^{m}\ p_{ij},
	\label{eq:rpg:policy}
\end{align}
where the $p_{ij}$ is defined in \eq{\ref{eq:rpgt:pij}}. The probability depends on the relative action values $q=[\lambda_1,...,\lambda_m]$. The highest relative action value
leads to the highest probability to be selected.
\label{def:rpg:policy}
\end{definition}

\begin{assumption}
For a state $s$, the set of events $E = \{e_{ij}| \forall i \neq j\}$ are conditionally independent, where 
$e_{ij}$ denotes the event that action $a_i$ is ranked higher than action $a_j$. The independence of the events is conditioned on
a MDP and a stationary policy.
\label{ass:rpgind}
\end{assumption}
\vspace{-0.5em}
Our ultimate goal is to maximize the long-term reward through optimizing the 
pairwise ranking policy or equivalently optimizing pairwise relationship among the
action pairs. Ideally, we would like the pairwise ranking policy selects the 
best action with the highest probability and the highest $\lambda$-value. 
To achieve this goal, we resort to the policy gradient method.
Formally, we propose the ranking policy gradient method (RPG), as shown in
Theorem~\ref{th:rpgt}.

\begin{theorem}[Ranking Policy Gradient Theorem]
For any MDP, the gradient of the expected long-term reward $J(\theta) = \sum_{\tau} p_{\theta}(\tau) r(\tau)$ w.r.t. the parameter $\theta$ of a pairwise ranking policy (Def~\ref{def:rpg:policy}) can be approximated by:
\begin{align}
\grad_{\theta} J({\theta}) \approx \vE_{\tau \sim \pi_\theta}\left[ \sum_{t=1}^{T}\nolimits\grad_{\theta} \left(\sum_{j=1, j\neq i}^{m} \nolimits(\lambda_i - \lambda_j)/2\right) r(\tau)\right],
\label{eq:rpg}
\end{align}
and the deterministic pairwise ranking policy $\pi_{\theta}$ is:
$a = \argmax_{i} \lambda_i, \ i=1,\dots,m$,
where $\lambda_i$ denotes the relative action value of action $a_i$ ($\lambda_{\theta}(s_t, a_t)$, $a_i = a_t$), 
$s_t$ and $a_t$ denotes the $t$-th state-action pair in trajectory $\tau$, 
$\lambda_j, \forall j\neq i$ denote the relative action values of all other actions that were not taken given state $s_t$ in trajectory $\tau$, i.e., $\lambda_{\theta}(s_t, a_j)$, $\forall a_j \neq a_t$.
\label{th:rpgt}
\end{theorem}

The proof of Theorem~\ref{th:rpgt} is provided in Appendix~\ref{sec:th:rpgt}.  Theorem~\ref{th:rpgt} states that optimizing the
discrepancy between the action values of the best action and all other
actions, is optimizing the pairwise relationships that maximize the return.
One limitation of RPG is that it is not convenient for the tasks where only
optimal stochastic policies exist since the pairwise ranking policy takes
extra efforts to construct a probability distribution [see Appendix~\ref{subsec:rpgvalid}].  
In order to learn the stochastic
policy, we introduce Listwise Policy Gradient (LPG) that optimizes the
probability of ranking a specific action on the top of a set of actions, with
respect to the return. In the context of RL, this top one probability is the
probability of  action $a_i$ to be chosen, which is equal to the sum of
probability all possible permutations that map action $a_i$ at the top.  This
probability is computationally prohibitive  since we need to consider the
probability of $m!$ permutations. Inspired by listwise learning to rank
approach~\citep{cao2007learning}, the top one probability can be modeled by the
softmax function (see Theorem~\ref{th:top1}). Therefore, LPG is equivalent to
the \textsc{Reinforce}~\citep{williams1992simple} algorithm with a softmax
layer.  LPG provides another interpretation of \textsc{Reinforce} algorithm
from the perspective of learning the optimal ranking and enables the learning
of both  deterministic policy and stochastic policy (see Theorem~\ref{th:lpgt}). 

\begin{theorem}[\citep{cao2007learning}, Theorem 6]
	Given the action values $q=[\lambda_1,...,\lambda_m]$, the probability of action $i$ to be chosen (i.e. to 
	be ranked on the top of the list) is:
\begin{align}
\pi(a_t=a_i|s_t)= \frac{\phi(\lambda_i)}{\sum_{j=1}^m {\phi(\lambda_j)}},
\label{eq:top1:p}	
\end{align}
where $\phi(*)$ is any increasing, strictly positive function. A common choice of $\phi$ is
the exponential function.
\label{th:top1}
\end{theorem}

\begin{theorem}[Listwise Policy Gradient Theorem]
For any MDP, the gradient of the long-term reward $J(\theta) = \sum_{\tau} p_{\theta}(\tau) r(\tau)$ w.r.t. 
the parameter $\theta$ of listwise ranking policy takes the following form:
\begin{align}
\grad_{\theta} J({\theta}) = \vE_{\tau \sim \pi_\theta}\left[ \sum_{t=1}^{T}\grad_{\theta} \left( \log \frac{e^{\lambda_i}}{\sum_{j=1}^m {e^{\lambda_j}}}\right) r(\tau)\right], 
\label{eq:lpg}
\end{align}
where the listwise ranking policy $\pi_{\theta}$ parameterized by $\theta$ is given by \eq{\ref{eq:lpgt:argmax}} 
for tasks with deterministic optimal policies:
\begin{align}
a = \argmax_{i} \lambda_i, \quad i=1,\dots,m
\label{eq:lpgt:argmax}
\end{align}
or \eq{\ref{eq:lpgt:sample}} for stochastic optimal policies:
\begin{align}
a \sim \pi(*|s), \quad i=1,\dots,m
\label{eq:lpgt:sample}
\end{align}
where the policy takes the form as in \eq{\ref{eq:lpgt:p}}
\begin{align}
\pi(a=a_i|s_t)= \frac{e^{\lambda_i}}{\sum_{j=1}^m {e^{\lambda_j}}}
\label{eq:lpgt:p}	
\end{align}
is the probability that action $i$ being ranked highest, given the current state and all the relative action values $\lambda_1 \dots \lambda_m$.
\label{th:lpgt}
\end{theorem}

The proof of Theorem~\ref{th:lpgt} exactly follows the direct policy differentiation~\citep{peters2008reinforcement,williams1992simple} by replacing the policy to the form of the Softmax function. 
The action probability $\pi(a_i|s), \forall i=1,...,m$ forms a probability distribution
over the set of discrete actions~\citet[Lemma 7]{cao2007learning}. 
Theorem~\ref{th:lpgt} states that the vanilla policy gradient~\citep{williams1992simple} 
parameterized by Softmax layer is optimizing the 
probability of each action to be ranked highest, with respect to the long-term reward. 
Furthermore, it enables learning both of the
deterministic policy and stochastic policy.

To this end, seeking sample-efficiency motivates us to learn the relative relationship
(RPG (Theorem~\ref{th:rpgt}) and LPG (Theorem~\ref{th:lpgt}))
of actions, instead of deriving policy based on action value estimations. 
However, both of the RPG and LPG belong to policy gradient methods, 
which suffers from large variance and the on-policy learning requirement~\citep{sutton2018reinforcement}.
Therefore, the intuitive implementations of RPG or LPG are still far from sample-efficient.
In the next section, we will describe a general off-policy learning framework 
empowered by supervised learning, which provides an alternative way to accelerate learning, 
preserve optimality, and reduce variance. 



\section{Off-policy learning as supervised learning} 
\label{sec:off_policy_learning_as_supervise_learning}

In this section, we discuss the connections and discrepancies between RL 
and supervised learning, and our results lead to a sample-efficient off-policy 
learning paradigm for RL. 
The main result in this section is Theorem~\ref{th:ltpt}, which casts the
problem of maximizing the lower bound of return into a supervised learning
problem, given one relatively mild Assumption~\ref{ass:uot} and practical
assumptions~\ref{ass:rpgind},\ref{ass:gmn}. It can be shown that these  
assumptions are valid in a range of common RL tasks, as discussed in 
Lemma~\ref{lemma:uot} in Appendix~\ref{sec:lemma:op}.
The central idea is to collect only the near-optimal trajectories
when the learning agent interacts with the environment, and imitate the
near-optimal policy by maximizing the log likelihood of the state-action pairs
from these near-optimal trajectories. With the road map in mind, we then begin to
introduce our approach as follows. 

In a discrete action MDP with finite states and horizon, given the
near-optimal policy $\pi_*$,  the stationary state distribution is given by:
$p_{\pi_*}(s) = \sum\nolimits_\tau p(s|\tau)p_{\pi_*}(\tau)$, where
$p(s|\tau)$ is the probability of a certain state given a specific trajectory
$\tau$ and is not associated with any policies, and only $p_{\pi_*}(\tau)$ is
related to the policy parameters.  The stationary distribution of
state-action pairs is thus: $p_{\pi_*}(s, a) = p_{\pi_*}(s)\pi_*(a|s)$. 
In this section, we consider the MDP that each initial state will lead to at least
one (near)-optimal trajectory. For a more general case, please refer to the
discussion in Appendix~\ref{sec:th:optimal}. In order to connect supervised
learning (i.e., imitating a near-optimal policy)  with RL and enable
sample-efficient off-policy learning, we first introduce the trajectory reward
shaping (TRS), defined as follows:
\begin{definition}[Trajectory Reward Shaping, TRS] 
Given a fixed trajectory $\tau$, 
its trajectory reward is shaped as follows:
\[
    w(\tau) = \left\{
                \begin{array}{ll}
             1, \ \text{if } r(\tau) \geq c  \\
                  0, \ o.w.
                \end{array}  \right.
\]
where $c = R_{\max} - \epsilon$ is a problem-dependent near-optimal trajectory reward threshold that 
indicates the least reward of near-optimal trajectory, $\epsilon \geq 0$ and $\epsilon\ll R_{\max}$. 
We denote the set of all possible near-optimal
trajectories as $\cT = \{\tau|w(\tau)=1\}$, i.e., $\ w(\tau)=1, \forall \tau \in \cT$. 
\label{def:PI}
\end{definition}


\begin{remark}
The threshold $c$ indicates a trade-off between the sample-efficiency and the optimality.
The higher the threshold, the less frequently it will hit the near-optimal 
trajectories during exploration, which means
it has higher sample complexity, while the final performance is better (see Figure~\ref{fig:oetf}).
\end{remark}

\begin{remark}
The trajectory reward can be reshaped to any positive functions that are not related to policy parameter $\theta$. For example, if we set $w(\tau)= r(\tau)$, the conclusions in this section still hold (see \eq{\ref{eq:constantw}} in Appendix~\ref{sec:th:ltpt}). For the sake of simplicity, we set $w(\tau) = 1$.
\end{remark}

Different from the reward shaping work~\citep{ng1999policy}, 
where shaping happens at each step on $r(s_t, a_t)$,
the proposed approach directly shapes the trajectory reward $r(\tau)$, 
which facilitates the smooth transform from RL to SL. 
After shaping the trajectory reward, we can transfer the goal of RL
from maximizing the return to maximize the long-term performance (Def~\ref{def:ltp}). 
\begin{definition}[Long-term Performance]
The long-term performance is defined by the expected shaped trajectory reward:
  \begin{align}
    \sum\nolimits_{\tau} p_\theta(\tau) w(\tau).
   \label{eq:ltpt:elp} 
  \end{align}
According to Def~\ref{def:PI}, the expectation over all trajectories 
is the equal to that over the near-optimal trajectories in $\cT$, i.e., 
$\sum\nolimits_{\tau} p_\theta(\tau) w(\tau)  = \sum\nolimits_{\tau\in \cT} p_\theta(\tau) w(\tau)$. 
\label{def:ltp}
\end{definition}

The optimality is preserved after trajectory reward shaping ($\epsilon =0, c = R_{\max}$) since
the optimal policy $\pi_*$ maximizing long-term performance is also an optimal policy for the original MDP, 
i.e., $\sum_{\tau}p_{\pi_*}(\tau)r(\tau) = \sum_{\tau \in \cT}p_{\pi_*}(\tau)r(\tau)=R_{\max}$, 
where $\pi_* = \argmax_{\pi_\theta}\sum_{\tau}p_{\pi_\theta}(\tau)w(\tau)$ and 
$p_{\pi_*}(\tau) = 0, \forall \tau \notin \cT$ (see Lemma~\ref{lemma:op} in Appendix~\ref{sec:th:ltpt}).
Similarly, when $\epsilon > 0$, the optimal policy after trajectory reward shaping
is a near-optimal policy for original MDP.
Note that most policy gradient methods use the softmax function, in which
we have $\exists \tau \notin \cT, p_{\pi_{\theta}}(\tau) > 0 $ (see Lemma~\ref{lemma:pp} in Appendix~\ref{sec:th:ltpt}).
Therefore when softmax is used to model a policy, it will not converge to an exact optimal policy. 
On the other hand, ideally, the discrepancy of the performance between them can be arbitrarily small based
on the universal approximation~\citep{hornik1989multilayer} with general conditions on the activation 
function and Theorem 1 in~\cite{syed2010reduction}.


Essentially, we use TRS to filter out near-optimal
trajectories and then we maximize the probabilities of near-optimal trajectories
to maximize the long-term performance. This procedure
can be approximated by maximizing the log-likelihood of near-optimal
state-action pairs, which is a supervised learning problem. Before we state
our main results, we first introduce the definition of uniformly near-optimal
policy (Def~\ref{def:uop}) and a prerequisite
(Asm.~\ref{ass:uot}) specifying the applicability of the
results.

\begin{definition}[Uniformly Near-Optimal Policy, UNOP]
  The Uniformly Near-Optimal Policy $\pi_*$ is the policy whose probability distribution over near-optimal trajectories ($\cT$) is a uniform distribution. i.e. $p_{\pi_*}(\tau) = \frac{1}{|\cT|}, \forall \tau \in \cT$, where $|\cT|$ is the 
  number of near-optimal trajectories. When we set $c=R_{\max}$, it is an 
  optimal policy in terms of both maximizing return and long-term performance. 
  In the case of $c=R_{\max}$, the corresponding uniform policy is an optimal policy, 
  we denote this type of optimal policy as uniformly optimal policy (UOP). 
  \label{def:uop}
\end{definition}
\begin{assumption}[Existence of Uniformly Near-Optimal Policy]
We assume the existence of Uniformly Near-Optimal Policy (Def.~\ref{def:uop}). 
\label{ass:uot}
\end{assumption}

Based on Lemma~\ref{lemma:uot} in Appendix~\ref{sec:lemma:op},
Assumption~\ref{ass:uot} is satisfied for certain MDPs that have deterministic dynamics. 
Other than Assumption~\ref{ass:uot}, all other assumptions in this work
(Assumptions~\ref{ass:rpgind},\ref{ass:gmn}) can almost always be
satisfied in practice, based on empirical observations. With these relatively
mild assumptions, we present the following long-term performance theorem, which shows 
the close connection between supervised learning and RL. 

\begin{theorem}[Long-term Performance Theorem]
Maximizing the lower bound of expected long-term performance in \eq{\ref{eq:ltpt:elp}}
is maximizing the log-likelihood of state-action pairs sampled from a
uniformly (near)-optimal policy $\pi_*$, which is a supervised learning problem:
  \begin{align}
  \argmax_{\theta} \sum\nolimits_{s \in \mathcal S} \sum_{a \in \mathcal A_s}\nolimits p_{\pi_*}(s, a) \log\pi_{\theta}(a|s) 
  \label{eq:ltpt:sl}
  \end{align}
The optimal policy of maximizing the lower bound is also the optimal policy of 
maximizing the long-term performance and the return.
\label{th:ltpt}
\end{theorem}
\begin{remark}
  It is worth noting that Theorem~\ref{th:ltpt} does not require a uniformly near-optimal policy $\pi_*$ to be 
  deterministic. The only requirement is the existence of a uniformly near-optimal policy. 
\end{remark}

\begin{remark}
  Maximizing the lower bound of long-term performance is maximizing the lower bound of long-term reward since we can set $w(\tau)=r(\tau)$ and $\sum_{\tau} p_{\theta}(\tau)r(\tau) \geq \sum_{\cT} p_{\theta}(\tau)w(\tau)$. An optimal policy that maximizes this lower bound is also an optimal policy maximizing the long-term performance when $c= R_{\max}$, thus maximizing the return. 
\end{remark}
The proof of Theorem~\ref{th:ltpt} can be found in Appendix~\ref{sec:th:ltpt}.  Theorem~\ref{th:ltpt} indicates that we break the
dependency between  current policy $\pi_\theta$ and the environment dynamics,
which means off-policy learning is able to be conducted by the above
supervised learning approach. Furthermore, we
point out that there is a potential discrepancy between imitating UNOP by 
maximizing log likelihood (even when the
optimal policy's samples are given) and the reinforcement learning since we
are maximizing a lower bound of expected long-term performance (or equivalently
the return over the near-optimal trajectories only) instead of return over all
trajectories. In practice, the state-action pairs from an optimal policy is hard
to construct while the uniform characteristic of UNOP can alleviate this issue (see Sec~\ref{sec:algorithm}). 
Towards sample-efficient RL, we apply Theorem~\ref{th:ltpt} to RPG, 
which reduces the ranking policy gradient to a classification problem by Corollary~\ref{cor:rppg}.
\begin{corollary}[Ranking performance policy gradient]
The lower bound of expected long-term performance (defined in~\eq{\ref{eq:ltpt:elp}}) using pairwise ranking policy (\eq{\ref{eq:rpg:policy}}) can be approximately optimized by the following loss:
\begin{align}
  \min_{\theta} \sum_{s, a_i}\nolimits p_{\pi_*}(s, a_i) \left(\sum_{j=1,j\neq i}^{m}\nolimits \max (0 , 1 + \lambda(s, a_j) - \lambda(s,  a_{i}))\right). \label{eq:rankingv2}
\end{align}
\label{cor:rppg}
\end{corollary}
\begin{corollary}[Listwise performance policy gradient]
Optimizing the lower bound of expected long-term performance by 
the listwise ranking policy (\eq{\ref{eq:lpgt:p}}) is equivalent to:
\begin{align}
\max_{\theta} \sum_{s}\nolimits p_{\pi_*}(s) \sum_{i=1}^m\nolimits \pi_*(a_i|s) \log \frac{e^{\lambda_i}}{\sum_{j=1}^m {e^{\lambda_j}}}
\label{eq:lppg:cross} 
\end{align}
The proof of this Corollary is a direct application of theorem~\ref{th:ltpt} by 
replacing policy with the softmax function. 
\label{cor:lppg}
\end{corollary}

The proof of Corollary~\ref{cor:rppg} can be found in Appendix~\ref{sbsec:corrppg}.  
Similarly, we can reduce LPG to a classification
problem (see Corollary~\ref{cor:lppg}).  One advantage of casting RL to SL
is variance reduction. With the proposed off-policy supervised learning,  we
can reduce the upper bound of the policy gradient variance, as shown in the
Corollary~\ref{cor:pgvr}. Before introducing the variance reduction results, we first make the  
common assumptions on the MDP regularity (Assumption~\ref{ass:gmn}) similar to~\citep[A1]{dai2017sbeed,degris2012off}. Furthermore, the Assumption~\ref{ass:gmn} is 
guaranteed for bounded continuously differentiable policy such as softmax function.

\begin{assumption}
  we assume the existence of maximum norm of log gradient over all possible state-action pairs, i.e.
$$C= \max_{s, a} \norm{\grad_{\theta}\log \pi_{\theta}(a|s)}_\infty$$
\label{ass:gmn}
\end{assumption}
\vspace{-2em}

\begin{corollary}[Policy gradient variance reduction]
Given a stationary policy, 
the upper bound of the variance of each dimension of policy gradient is
$\cO(T^2C^2R^2_{\max})$. The upper bound of gradient variance of maximizing the
lower bound of long-term performance  \eq{\ref{eq:ltpt:sl}} is $\cO(C^2)$, where
$C$ is the maximum norm of log gradient based on Assumption~\ref{ass:gmn}. The supervised learning has reduced the upper bound of gradient variance by an order of $\cO(T^2 R_{\max}^2)$ as compared to the regular policy gradient, 
considering $R_{\max}\geq 1, T\geq 1$,  which is a very common situation in practice.
\label{cor:pgvr}
\end{corollary}

The proof of Corollary~\ref{cor:pgvr} can be found in
Appendix~\ref{sec:app:pgvr}. This corollary shows that the
variance of regular policy gradient is upper-bounded by the square of time horizon
and the maximum trajectory reward. It is aligned with our intuition and empirical
observation: the longer the horizon the harder the learning. Also, the common
reward shaping tricks such as truncating the reward to $[-1, 1]$~\citep{dopamine} 
can help the learning
since it reduces variance by decreasing $R_{\max}$. With supervised learning,
we concentrate the difficulty of long-time horizon into the exploration
phase, which is an inevitable issue for all RL
algorithms, and we drop the dependence on $T$ and $R_{\max}$ for policy variance.
Thus, it is more stable and efficient to train the policy using supervised
learning.
One potential limitation of this method is that the trajectory reward threshold $c$ is
task-specific, which is crucial to the final performance and sample-efficiency. 
In many applications such as Dialogue system~\citep{li2017end}, recommender system~\citep{melville2011recommender}, etc.,
we design the reward function to guide the learning process, in which $c$ is naturally known.
For the cases that we have no prior knowledge on the reward function of MDP, we treat $c$ as a tuning
parameter to balance the optimality and efficiency, as we empirically verified in Figure~\ref{fig:oetf}.
The major theoretical uncertainty on general tasks is the existence of 
a uniformly near-optimal policy, which is negligible to the empirical performance. 
The rigorous theoretical analysis of this problem is beyond the scope of this work. 

\section{An algorithmic framework for off-policy learning}
\label{sec:algorithm}

\begin{figure*}
\centering
\includegraphics[width=0.9\textwidth]{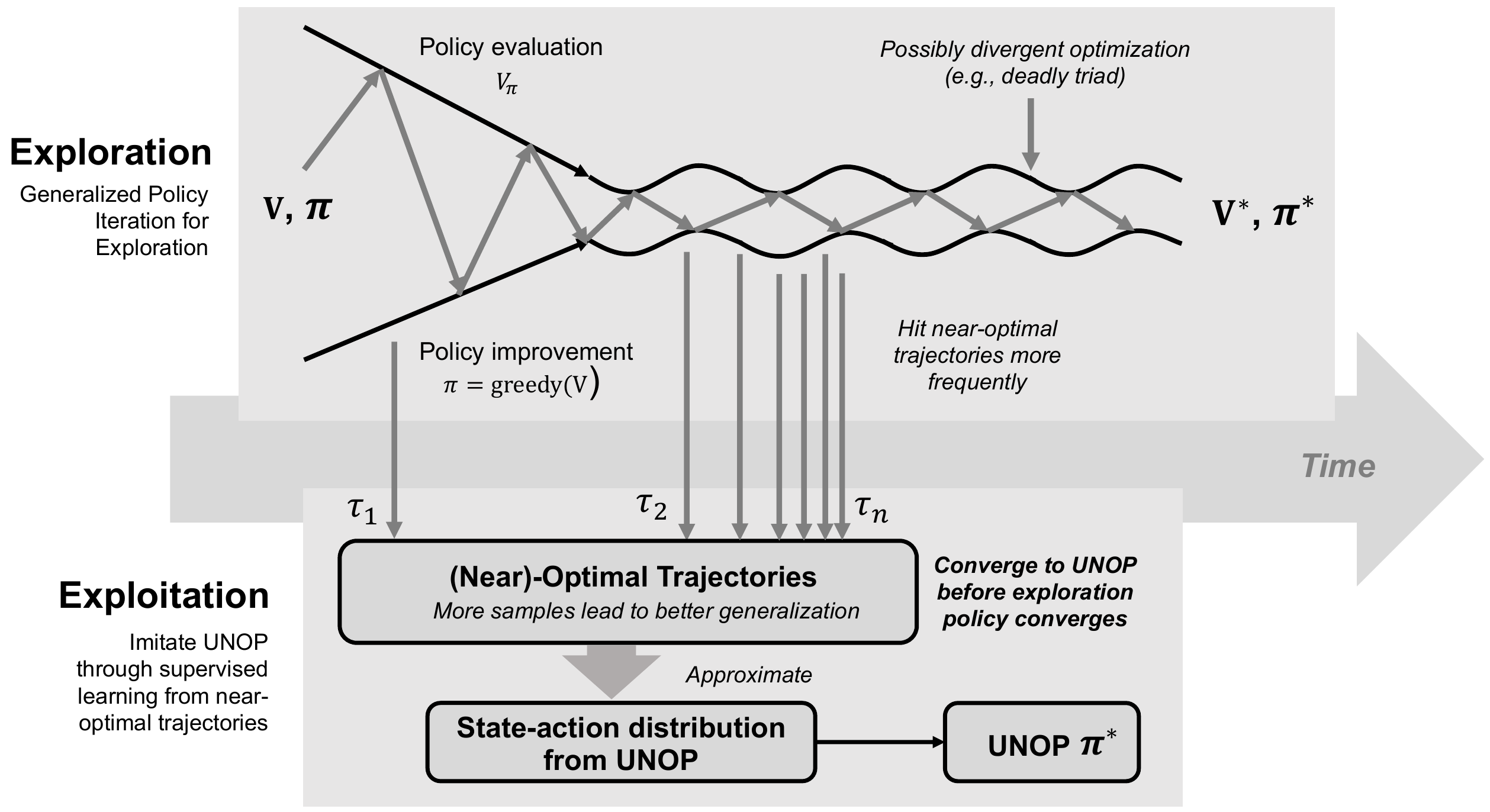}
\vspace{-1em} 
	\caption{The off-policy learning as supervised learning framework for general policy gradient methods.}
\vspace{-2em}
\label{fig:gpioff}
\end{figure*}

Based on the discussions in Section~\ref{sec:off_policy_learning_as_supervise_learning},
we exploit the advantage of reducing RL into supervised learning via a proposed
two-stages off-policy learning framework. As we illustrated in Figure~\ref{fig:gpioff},
the proposed framework contains the following two stages: 

\textbf{Generalized Policy Iteration for Exploration.}
The goal of the exploration stage is to collect different near-optimal
trajectories as frequently as possible. Under the off-policy framework, the
exploration agent and the learning agent can be separated. Therefore, any
existing RL algorithm can be used during the exploration. The
principle of this framework is using the most advanced RL agents
as an exploration strategy in order to collect more near-optimal trajectories
and leave the policy learning to the supervision stage.

\textbf{Supervision.} In this stage, we imitate the uniformly near-optimal policy, UNOP (Def~\ref{def:uop}). Although we have no access to the UNOP,
we can approximate the state-action distribution from UNOP by collecting the near-optimal trajectories only. The near-optimal samples are constructed online 
and we are not given any expert demonstration or expert policy beforehand. 
This step provides a sample-efficient approach to conduct exploitation, which 
enjoys the superiority of stability (Figure~\ref{fig:pong}), variance reduction (Corollary \ref{cor:pgvr}), and optimality preserving (Theorem~\ref{th:ltpt}).

The two-stage algorithmic framework can be directly incorporated in RPG and LPG to 
improve sample efficiency. The implementation of RPG is given in Algorithm~\ref{alg:rankingrl}, 
and LPG follows the same procedure except for the difference in the loss function. 
The main requirement of Alg.~\ref{alg:rankingrl} is on the exploration efficiency and
the MDP structure. During the exploration stage, a sufficient amount of the different near-optimal 
trajectories need to be collected for constructing a representative supervised learning training dataset.
Theoretically, this requirement always holds [see Appendix
Section~\ref{sec:lemma:op}, Lemma~\ref{lemma:hot}], while the number of episodes explored
could be prohibitively large, which makes this algorithm sample-inefficient. 
This could be a practical concern of the proposed algorithm. However, 
according to our extensive empirical observations, 
we notice that long before the value function based state-of-the-art 
converges to near-optimal performance, enough amount of near-optimal trajectories are already explored.

Therefore, we point out that instead of estimating optimal action value functions and then
choosing action greedily, 
using value function to facilitate the exploration and imitating UNOP is a
more sample-efficient approach.  As illustrated in Figure~\ref{fig:gpioff},
value based methods with off-policy learning, bootstrapping, and function
approximation could lead to a divergent optimization~\citep[Chap.
11]{sutton2018reinforcement}. In contrast to resolving the instability,  
we circumvent this issue via constructing a
stationary target using the samples from (near)-optimal trajectories, and
perform imitation learning.   This two-stage approach can avoid the extensive
exploration of the suboptimal state-action space and reduce the substantial
number of samples needed for estimating optimal action values. In the MDP
where we have a high probability of hitting the near-optimal trajectories (such as
\textsc{Pong}), the supervision stage can further facilitate the exploration. 
It should be emphasized that our work
focuses on improving the sample-efficiency through more effective exploitation,
rather than developing novel exploration method. 

\begin{algorithm}[h!]\small
\begin{algorithmic}[1]
\REQUIRE The near-optimal trajectory reward threshold $c$, the number of maximal training episodes $N_{max}$.
Maximum number of time steps in each episode $T$, and batch size $b$.  
\WHILE { episode $< N_{\max}$}
\REPEAT
\STATE Retrieve state $s_t$ and sample action $a_t$ by the specified exploration agent 
       (can be random, $\epsilon$-greedy, or any RL algorithms).
\STATE Collect the experience $e_t = (s_t, a_t, r_t, s_{t+1})$ and store to the replay buffer.
\STATE $t= t +1$
\IF {t \% update step == 0}
\STATE Sample a batch of experience $\{e_j\}_{j=1}^b$ from the near-optimal replay buffer.
\STATE Update $\pi_{\theta}$ based on the hinge loss~\eq{\ref{eq:rankingv2}} for RPG.
\STATE Update the exploration agent using samples from the regular replay buffer (In simple MDPs such as 
\textsc{Pong} where near-optimal trajectories are encountered frequently, near-optimal replay buffer can be used to update the exploration agent).
\ENDIF
\UNTIL{terminal $s_t$ or $t- t_{\text{start}} >=T$} 
\IF {return $\sum_{t=1}^Tr_t \geq c$}
\STATE Take the near-optimal trajectory $e_t, t=1,...,T$ in the latest episode 
from the regular replay buffer, and insert the trajectory into the near-optimal replay buffer.
\ENDIF
\IF {$t$ \% evaluation step == 0}
\STATE Evaluate the RPG agent by greedily choosing the action. If the best performance is reached, then stop training.
\ENDIF
\ENDWHILE
\end{algorithmic}
\caption{Off-Policy Learning for Ranking Policy Gradient (RPG)}
\label{alg:rankingrl}
\end{algorithm}


\section{Sample Complexity and Generalization Performance}
\label{sec:samcomplex}
In this section, we present a theoretical analysis on the sample complexity of
RPG with off-policy learning framework in Section~\ref{sec:algorithm}. 
The analysis leverages the results from the Probably Approximately Correct (PAC) framework, 
and provides an alternative approach to quantify sample complexity of RL 
from the perspective of the connection between RL and SL (see Theorem~\ref{th:ltpt}),
which is significantly different from the existing approaches that use value function estimations~\citep{kakade2003sample,strehl2006pac,kearns2000approximate,strehl2009reinforcement,krishnamurthy2016pac,jiang2017contextual,jiang2018open,zanette2019tighter}. We show that the sample complexity of RPG (Theorem~\ref{th:rlgp})
depends on the properties of MDP such as horizon, action space, dynamics, and the 
generalization performance of supervised learning. It is worth mentioning that the sample
complexity of RPG has no linear dependence on the state-space, which makes it suitable for large-scale MDPs. 
Moreover, we also provide a formal quantitative definition (Def~\ref{def:ee}) on 
the exploration efficiency of RL. 

Corresponding to the two-stage framework in Section~\ref{sec:algorithm}, 
the sample complexity of RPG also splits into two problems:
\begin{itemize}
	\item \textbf{Learning efficiency:} How many state-action pairs from the uniformly optimal policy do we need to collect, in order to achieve good generalization performance in RL?
	\item \textbf{Exploration efficiency:} For a certain type of MDPs, what is the 
	probability of collecting $n$ training samples (state-action pairs from the 
	uniformly near-optimal policy) in the first $k$ episodes in the worst case? 
	This question leads to a quantitative evaluation metric of different exploration 
	methods. 
\end{itemize}
The first stage is resolved by Theorem~\ref{th:rlgp}, which connects
the lower bound of the generalization
performance of RL to the supervised learning generalization performance. 
Then we discuss the exploration efficiency of the worst case performance
for a binary tree MDP in Lemma~\ref{lemma:osr}. Jointly, we show how 
to link the two stages to give a general theorem that studies how many 
samples we need to collect in order to achieve certain performance in RL.

In this section, we restrict our discussion on the MDPs with a fixed action 
space and assume the existence of deterministic optimal policy. The policy
$\pi = \hat{h} = \argmin_{h \in \cH} \hat{\epsilon}(h)$ corresponds to the
empirical risk minimizer (ERM) in the learning theory literature, which 
is the policy we obtained through learning on the training samples. 
$\cH$ denotes the hypothesis class from where we are selecting the 
policy. Given a hypothesis (policy) $h$, the empirical risk is given by 
$\hat{\epsilon}(h) = \sum_{i=1}^n \frac{1}{n} \textbf{1}\{ h(s_i) \neq a_i\}$. 
Without loss
of generosity, we can normalize the reward function to set the upper bound of trajectory reward 
equals to one ($i.e., R_{\max} = 1$), similar
to the assumption in~\citep{jiang2018open}. 
It is worth noting that the training samples are generated \emph{i.i.d.} from an unknown
distribution, which is perhaps the most important assumption in the statistical 
learning theory. \emph{i.i.d.} is satisfied in this case since the state action pairs (training samples)
are collected by filtering the samples during the learning stage, and we can 
manually manipulate the samples to follow the distribution of UOP (Def~\ref{def:uop}) 
by only storing the unique near-optimal trajectories. 

\subsection{Supervision stage: Learning efficiency}
To simplify the presentation, we restrict our discussion on the finite
hypothesis class (i.e. $|\cH| < \infty$) since this dependence is
not germane to our discussion. However, we note that the theoretical framework in this section
is not limited to the finite hypothesis class. 
For example, we can simply use the VC dimension~\citep{vapnik2006estimation} or the Rademacher
complexity~\citep{bartlett2002rademacher} to generalize our discussion to the infinite hypothesis class, such as
neural networks. For completeness, we first revisit the sample complexity result
from the PAC learning in the context of RL.  
\begin{lemma}[Supervised Learning Sample Complexity~\citep{mohri2018foundations}]
	Let $|\cH| < \infty$, and let $\delta, \gamma$ be fixed, the inequality
	$\epsilon(\hat{h}) \leq  (\min_{h\in\cH}\epsilon(h))  + 2\gamma = \eta$ holds with probability at least
    $1 - \delta$, when the training set size $n$ satisfies:
    \begin{align}
    	n \geq \frac{1}{2\gamma^2} \log\frac{2|\cH|}{\delta},
    \end{align}
    \label{lemma:samplecomplexiy}
    where the generalization error (expected risk) of a hypothesis $\hat{h}$ is defined as:
\begin{align*}
\epsilon(\hat{h}) = \sum_{s, a}\nolimits p_{\pi_*}(s,a) \textbf{1} \left\{ \hat{h}(s) \neq a\right\}.
\end{align*}
\end{lemma}

\begin{condition}[Action values]
	We restrict the action values of RPG in certain range, i.e., $\lambda_i \in [0, c_q]$, where $c_q$ is a positive constant. 
	 \label{ass:av} 
\end{condition}
This condition can be easily satisfied, for example, 
we can use a sigmoid to cast the action values into $[0, 1]$. We can 
impose this constraint since in RPG we only focus on the relative relationship of 
action values. Given the mild condition and established on the prior work
in statistical learning theory, we introduce the following results that
connect the supervised learning and reinforcement learning.

\begin{theorem}[Generalization Performance]
Given a MDP where the UOP (Def~\ref{def:uop}) is deterministic, let $|\cH|$ denote the size of hypothesis space, and $\delta, n$ be fixed, 
the following inequality holds with probability at least $1 - \delta$:
\begin{align*}
& \sum_{\tau}\nolimits p_{\theta}(\tau)r(\tau) \geq  D(1 + e)^{\eta(1-m)T},
\end{align*}
where $D= |\cT|\left(\Pi_{\tau\in \cT} p_d(\tau)\right)^{\frac{1}{|\cT|}}$, $p_d(\tau) = p(s_1)\Pi_{t=1}^{T}p(s_{t+1}|s_{t}, a_{t})$ denotes the environment dynamics. $\eta$ is the upper bound of supervised learning generalization performance, defined as $\eta = (\min_{h\in\cH}\epsilon(h))  + 2\sqrt{\frac{1}{2n}\log \frac{2|\cH|}{\delta}} = 2\sqrt{\frac{1}{2n}\log \frac{2|\cH|}{\delta}} $.
\label{th:rlgp}
\end{theorem}
 

\begin{corollary}[Sample Complexity]
Given a MDP where the UOP (Def~\ref{def:uop}) is deterministic, let $|\cH|$ denotes the size of 
hypothesis space, and let $\delta$ be fixed.
Then for the following inequality to hold with probability at least $1 - \delta$:
\begin{align*}
& \sum\nolimits_{\tau} p_{\theta}(\tau)r(\tau)  \geq 1 - \epsilon,
\end{align*}
it suffices that the number of state action pairs (training sample size $n$) 
from the uniformly optimal policy satisfies:
    \begin{equation}
    	n \geq \frac{2(m-1)^2T^2}{(\log_{1+e}\frac{D}{1- \epsilon})^2} \log\frac{2|\cH|}{\delta}
    	= \mathcal O\left(\frac{m^2T^2}{\left(\log\frac{D}{1- \epsilon}\right)^2} \log\frac{|\cH|}{\delta}\right). \nonumber
    \end{equation}
\label{coro:rlgp}
\end{corollary}

The proofs of Theorem~\ref{th:rlgp} and Corollary~\ref{coro:rlgp} 
are provided in Appendix~\ref{sec:app:subthrlgp}. Theorem~\ref{th:rlgp}
establishes the connection between the generalization performance of RL and the sample 
complexity of supervised learning. The lower bound of generalization performance
decreases exponentially with respect to the horizon $T$ and action space
dimension $m$. This is aligned with our empirical observation that it is more
difficult to learn the MDPs with a longer horizon and/or a larger action space. 
Furthermore, the generalization performance has a linear dependence on $D$, the
transition probability of optimal trajectories. Therefore, $T$, $m$, and $D$
jointly determines the difficulty of learning of the given MDP. As pointed out by
Corollary~\ref{coro:rlgp}, the smaller the $D$ is,  the higher the sample
complexity. Note that $T$, $m$, and $D$ all characterize intrinsic properties of MDPs, which cannot be improved by our learning algorithms. 
One advantage of RPG is that its sample complexity has no dependence on the
state space, which enables the RPG to resolve large-scale complicated MDPs,
as demonstrated in our experiments.  In the supervision stage, our goal is the same as
in the traditional supervised learning: to achieve better generalization performance
$\eta$. 


\subsection{Exploration stage: Exploration efficiency}

The exploration efficiency is highly related to the MDP properties and the
exploration strategy. To provide interpretation on how the MDP properties
(state space dimension, action space dimension, horizon) affect the sample
complexity through exploration efficiency, we characterize a simplified MDP as in~\citep{sun2017deeply}
, in which we explicitly compute the exploration efficiency of a stationary policy
(random exploration), as shown in Figure~\ref{fig:mdps}. 

\begin{definition}[Exploration Efficiency]
	We define the exploration efficiency of a certain exploration algorithm ($A$) 
	within a MDP ($\cM$) as the probability of 
	sampling $i$ distinct optimal trajectories in the first $k$ episodes.
	We denote the exploration efficiency as $p_{A, \cM}(n_{traj} \geq i|k)$.
    When $\cM$, $k$, $i$ and optimality threshold $c$ are fixed, 
    the higher the $p_{A, \cM}(n_{traj} \geq i|k)$, the 
    better the exploration efficiency. We use 
    $n_{traj}$ to denote the number of near-optimal trajectories in this subsection. 
    If the exploration algorithm derives a series of learning policies, then we have
    $p_{A, \cM}(n_{traj} \geq i|k) = p_{\{\pi_i\}_{i=0}^t, \cM}(n_{traj} \geq i|k) $, where $t$
    is the number of steps the algorithm $A$ updated the policy. 
    If we would like to study the exploration efficiency of a stationary policy, then we have 
    $p_{A, \cM}(n_{traj} \geq i|k)= p_{\pi, \cM}(n_{traj} \geq i|k)$. 
    \label{def:ee}
\end{definition}

\begin{definition}[Expected Exploration Efficiency]
The expected exploration efficiency of a certain exploration algorithm ($A$) 
within a MDP ($\cM$) is defined as:
$$E_{A, k, \cM} = \sum\nolimits_{i=0}^k p_{A, \cM}(n_{traj} = i|k) i.$$ 
\label{def:eee}	
\end{definition}

The definitions provide a quantitative metric to evaluate 
the quality of exploration. Intuitively, the quality of exploration 
should be determined by how frequently it will hit different 
good trajectories. We use Def~\ref{def:ee} for theoretical analysis and
Def~\ref{def:eee} for practical evaluation.

\begin{figure*}
\centering
	\includegraphics[width=0.5\textwidth]{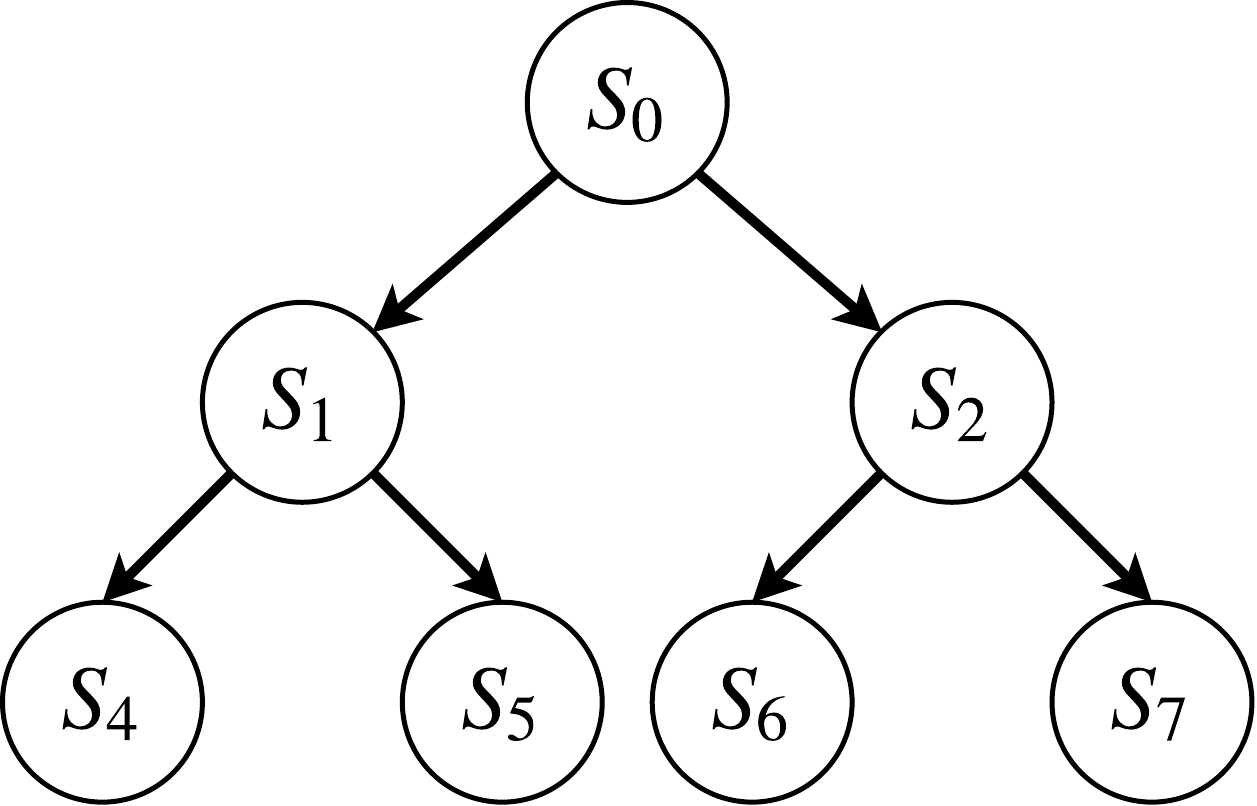} 
\caption{
The binary tree structure MDP ($\cM_1$) with one initial state, similar as discussed
in~\citep{sun2017deeply}.
In this subsection, we focus on the MDPs that have no duplicated states. The initial state distribution of the MDP is uniform and the environment dynamics is deterministic. 
For $\cM_1$ the worst case exploration is random exploration and each trajectory
will be visited at same probability under random exploration. Note that
in this type of MDP, the Assumption~\ref{ass:uot} is satisfied.}
\label{fig:mdps}
\end{figure*}

\begin{lemma}[The Exploration Efficiency of Random Policy]
The Exploration Efficiency of random exploration policy 
in a binary tree MDP ($\cM_1$) is given as: 
\begin{align*}
	p_{\pi_r, \cM}(n_{traj} \geq i|k)  = 1- \sum_{i'=0}^{i-1}\nolimits C_{|\cT|}^{i'} \frac{\sum_{j=0}^{i'}(-1)^j C_{i'}^j( N - |\cT| + i' -j)^k }{N^k},
\end{align*}
where $N$ denotes the total number of different trajectories in the MDP. 
In binary tree MDP $\cM_1$, $N = |\cS_0||\cA|^T$, where the $|\cS_0|$ denotes 
the number of distinct initial states.
$|\cT|$ denotes the number of optimal trajectories. $\pi_r$ denotes the 
random exploration policy, which means the probability of hitting each trajectory
in $\cM_1$ is equal. 
\label{lemma:osr}
\end{lemma}
The proof of Lemma~\ref{lemma:osr} is available in Appendix~\ref{sec:app:sublemmaosr}.

\subsection{Joint Analysis Combining Exploration and Supervision}
In this section, we jointly consider the learning efficiency and exploration efficiency 
to study the generalization performance. 
Concretely, we would like to study if we interact with 
the environment a certain number of episodes, 
what is the worst generalization performance we can expect with certain probability,
if RPG is applied. 
\begin{corollary}[RL Generalization Performance]
Given a MDP where the UOP (Def~\ref{def:uop}) is deterministic, let $|\cH|$ be the size of the hypothesis space, and let $\delta, n, k$ be fixed, 
the following inequality holds with probability at least
$1 - \delta'$: 
\begin{align*}
& \sum_{\tau}\nolimits p_{\theta}(\tau)r(\tau) \geq D(1 + e)^{\eta(1-m)T},
\end{align*}
where $k$ is the number of episodes we have explored in the MDP,
$n$ is the number of distinct optimal state-action pairs we needed from the UOP (i.e., size of training data.). $n'$ denotes the number of distinct optimal state-action pairs
collected by the random exploration. $\eta = 2\sqrt{\frac{1}{2n}\log \frac{2|\cH|p_{\pi_r, \cM}(n' \geq n |k)}{p_{\pi_r, \cM}(n' \geq n |k) - 1 + \delta'}}$. 
\label{coro:esrlgp}
\end{corollary}
The proof of Corollary~\ref{coro:esrlgp} is provided in
Appendix~\ref{sec:app:eerlgp}. 
Corollary~\ref{coro:esrlgp} states that the probability 
of sampling optimal trajectories is the main bottleneck of exploration and 
generalization, instead of state space dimension. 
In general, the optimal exploration strategy depends
on the properties of MDPs. In this work, we focus on improving learning 
efficiency, i.e., learning optimal ranking instead of estimating value functions.
The discussion of optimal exploration is beyond the scope of this work.

\section{Experimental Results}
\label{sec:exp}
\begin{figure*}[t!]
\centering
\begin{tabular}{ccc}
\hspace{-2em}	\includegraphics[width=0.33\textwidth]{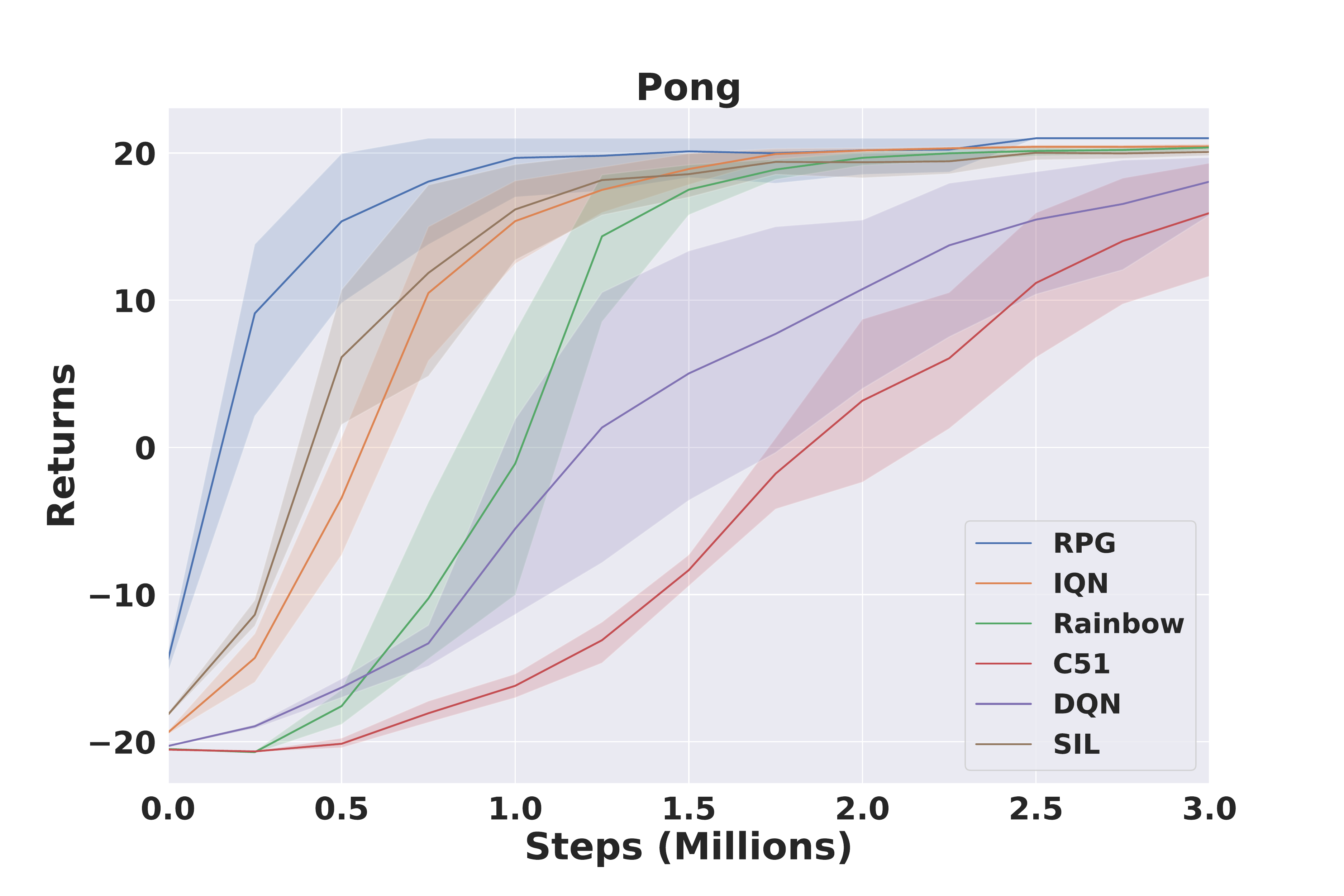} &
	\includegraphics[width=0.33\textwidth]{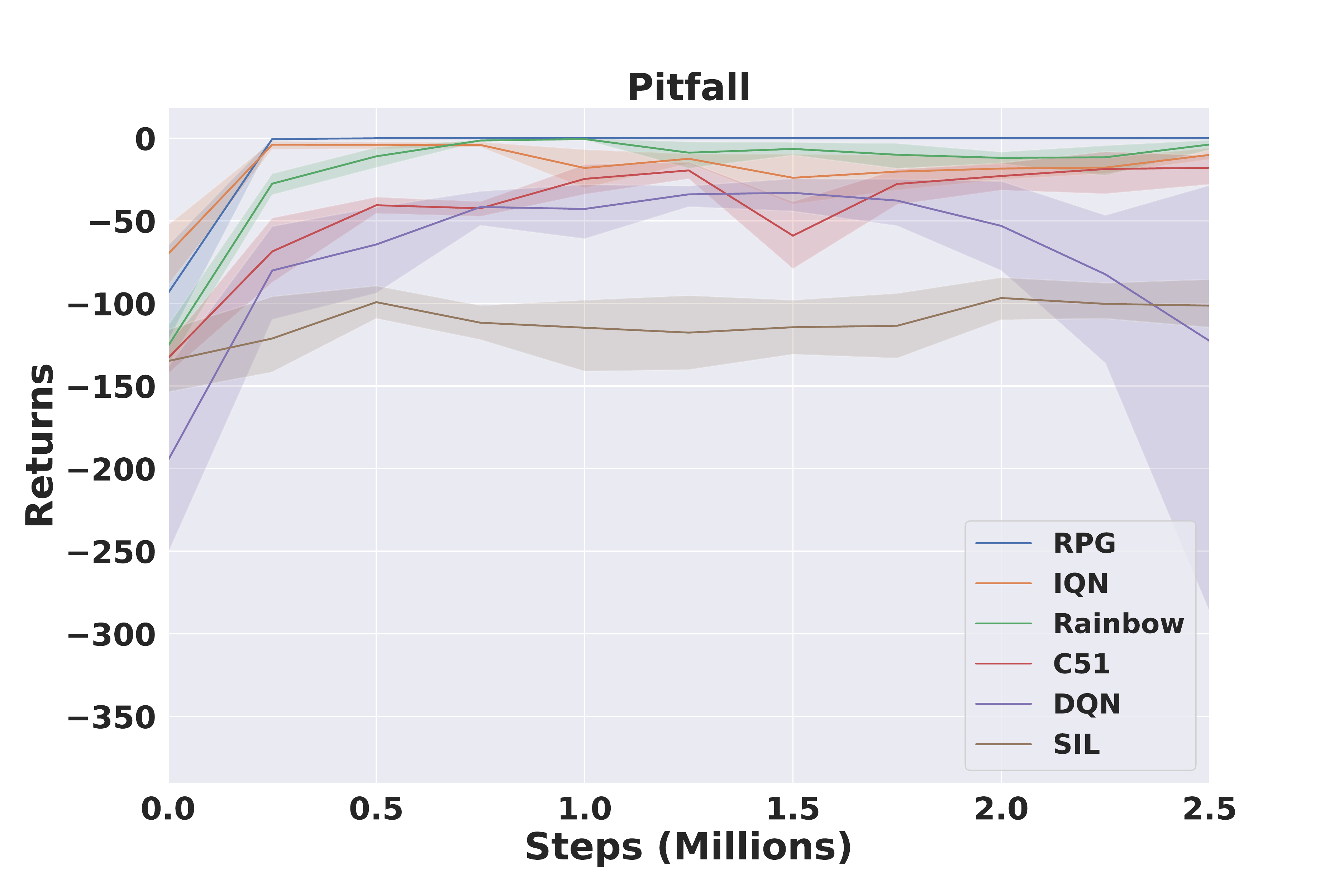} & 
	\includegraphics[width=0.33\textwidth]{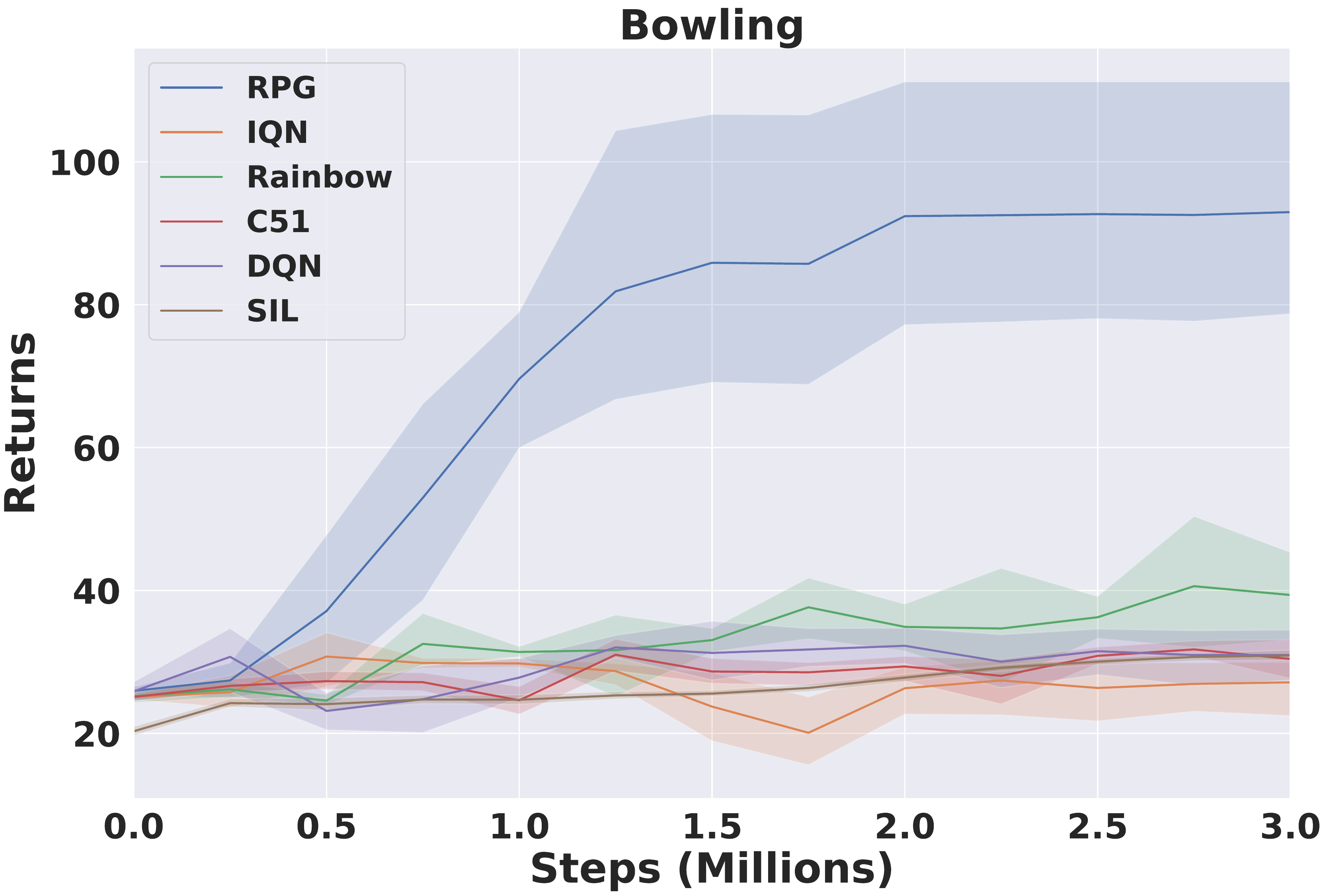} \\
\hspace{-2em}\includegraphics[width=0.33\textwidth]{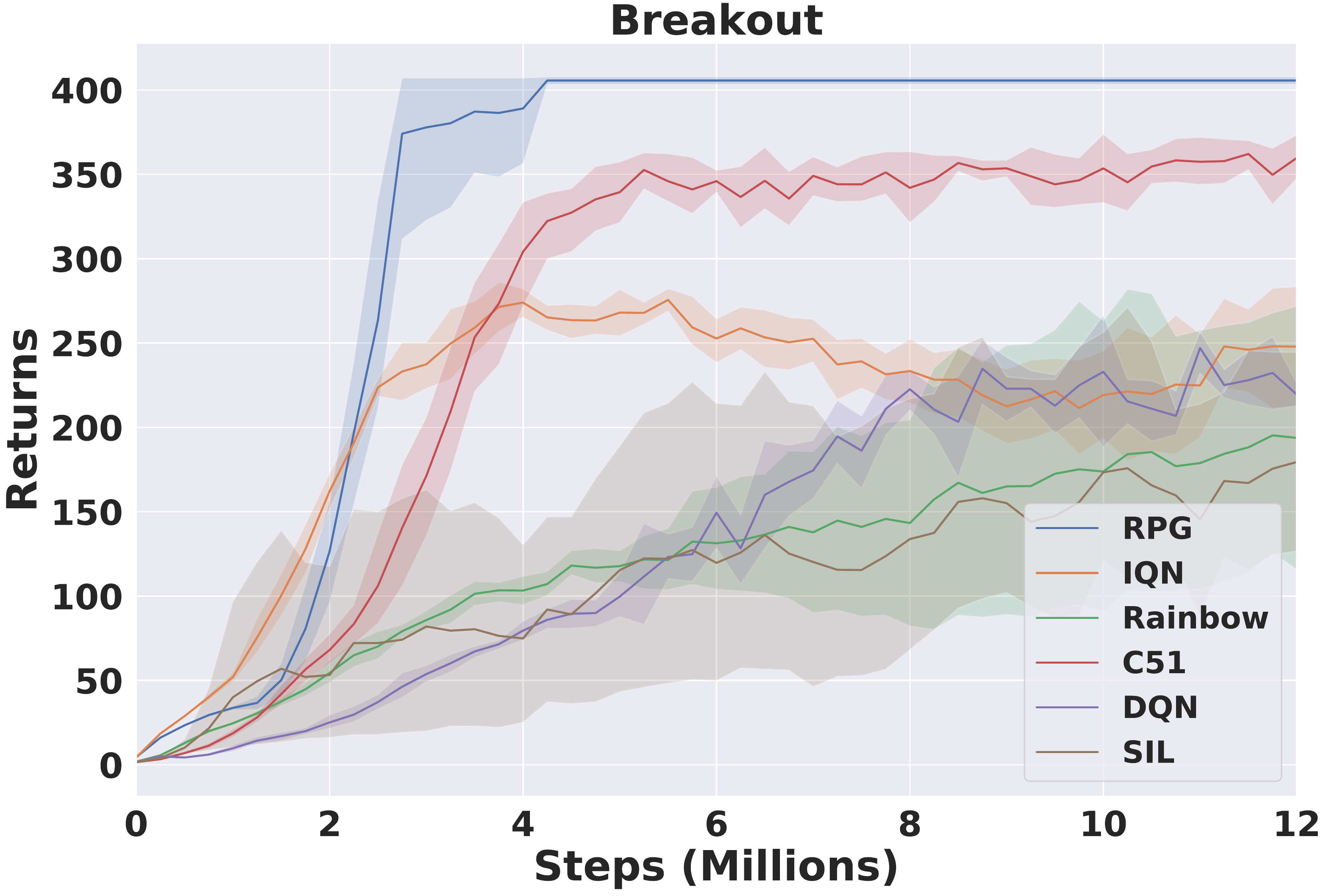} & 
	\includegraphics[width=0.33\textwidth]{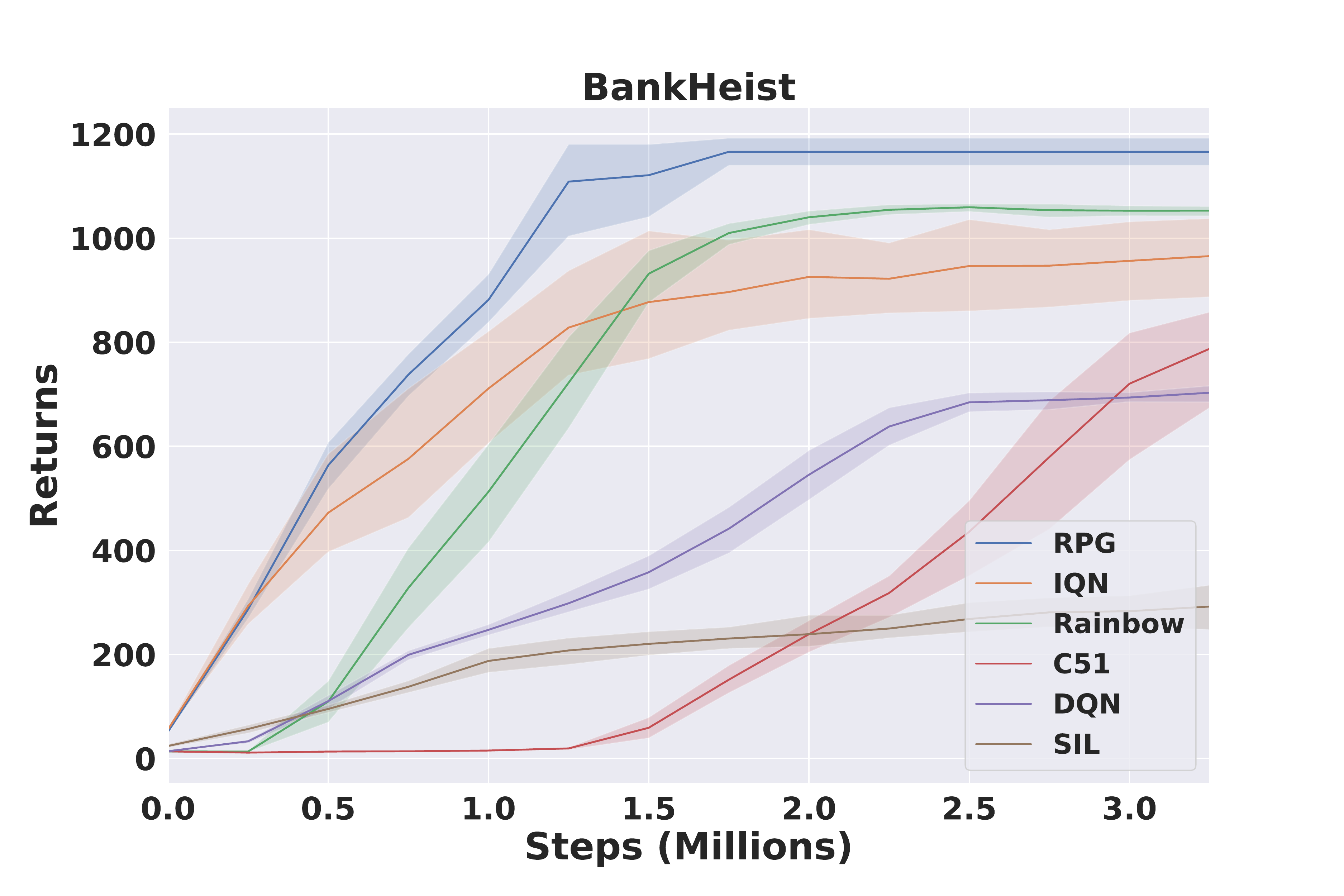} &
	\includegraphics[width=0.33\textwidth]{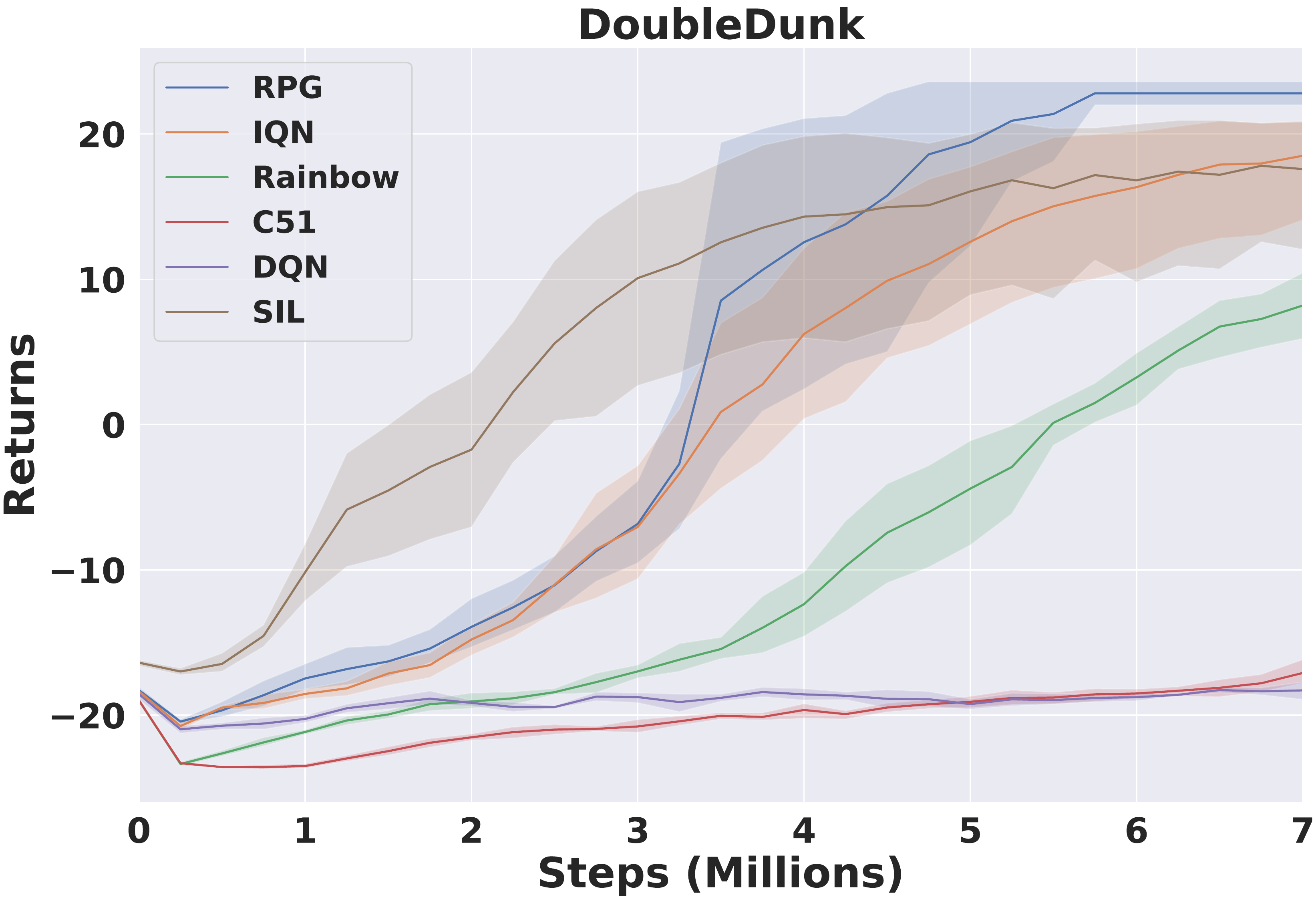} \\ 
	\hspace{-2em}\includegraphics[width=0.33\textwidth]{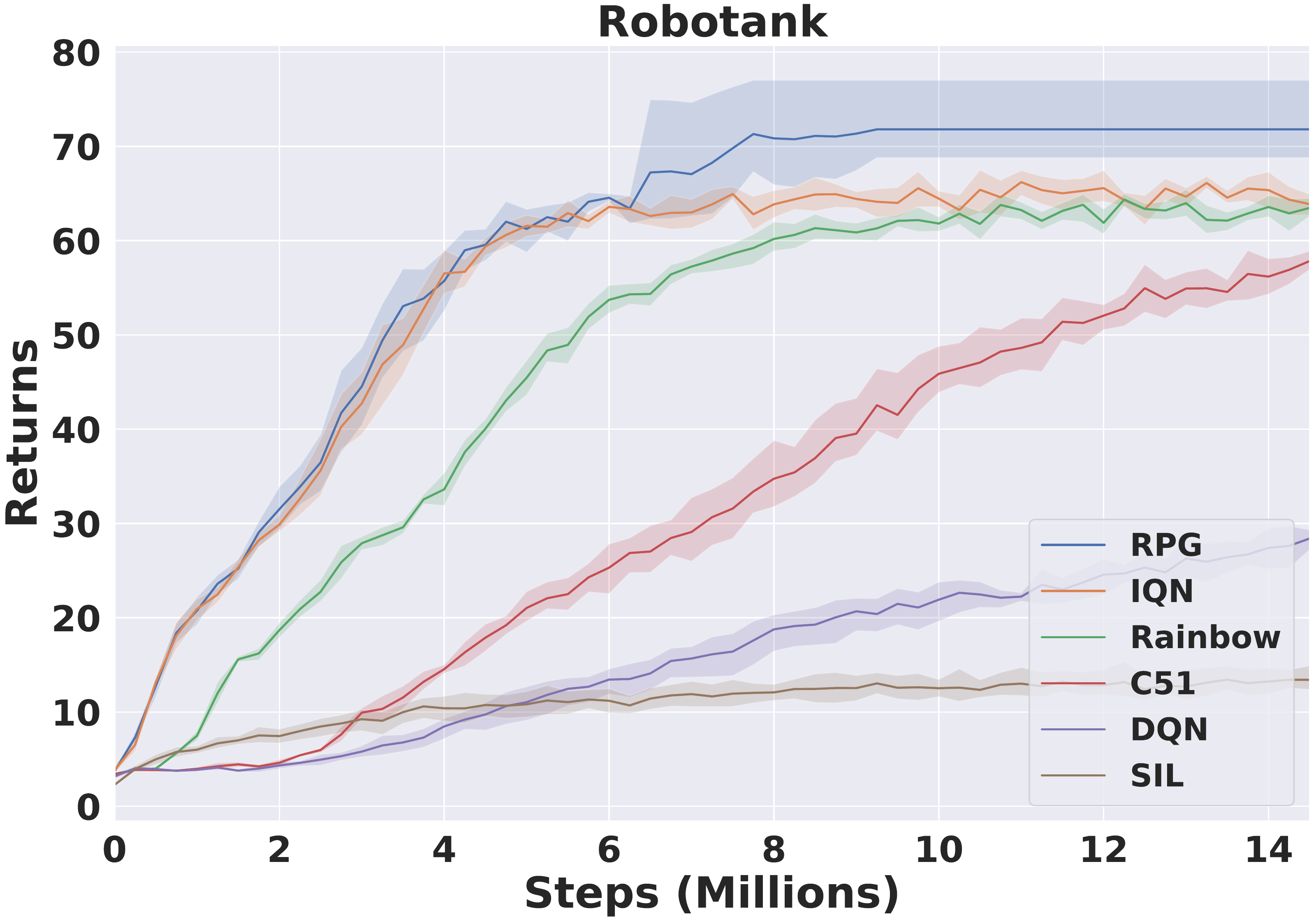} &
	\includegraphics[width=0.33\textwidth]{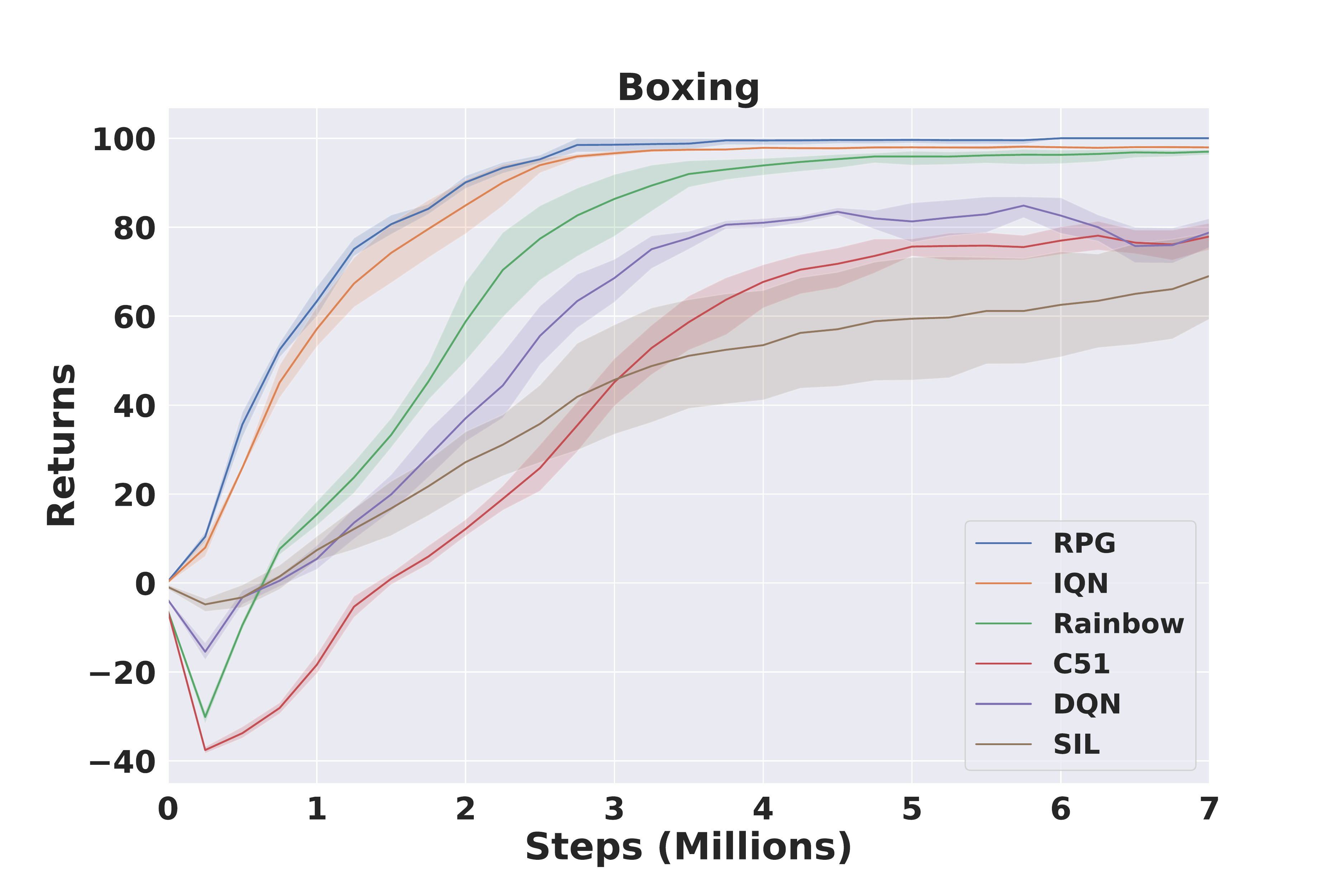} &
	\includegraphics[width=0.33\textwidth]{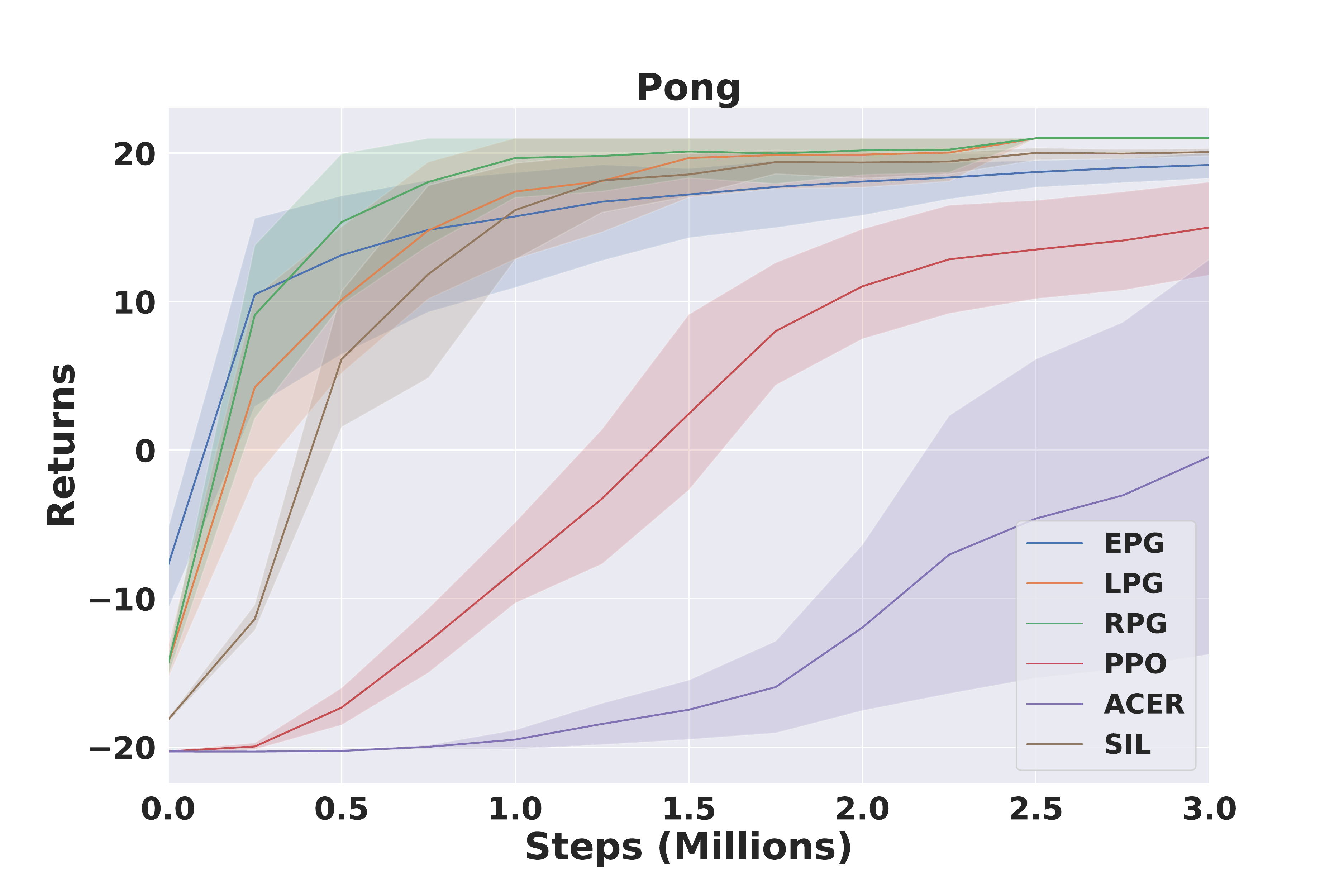}\\
\end{tabular}
\caption{The training curves of the proposed RPG and state-of-the-art. 
All results are averaged over random seeds from 1 to 5. The $x$-axis represents the number of steps interacting with the environment (we update the model every four steps)
and the $y$-axis represents the averaged training episodic return. 
The error bars are plotted with a confidence interval of 95\%.
}
\label{fig:pong}
\end{figure*}

To evaluate the sample-efficiency of Ranking Policy Gradient (RPG), we focus
on Atari 2600 games in OpenAI
gym~\citep{bellemare2013arcade,brockman2016openai}, without randomly repeating
the previous action. We compare our method with the state-of-the-art baselines
including DQN~\citep{mnih2015human}, C51~\citep{bellemare2017distributional},
IQN~\citep{dabney2018implicit}, \textsc{Rainbow}~\citep{hessel2017rainbow}, and self-imitation learning (SIL)~\citep{oh2018self}.
For reproducibility, we use the implementation provided in Dopamine
framework\footnote{https://github.com/google/dopamine}~\citep{dopamine} for all baselines and proposed methods, except for SIL using
the official implementation.~\footnote{https://github.com/junhyukoh/self-imitation-learning}. Follow the
standard practice~\citep{oh2018self,hessel2017rainbow,dabney2018implicit,bellemare2017distributional},
we report the training performance of all baselines as the increase of
interactions with the environment, or proportionally the number of training
iterations. We run the algorithms with five random seeds and report the
average rewards with $95$\% confidence intervals. The implementation details
of the proposed RPG and its variants are given as follows\footnote{Code is available at \href{https://github.com/illidanlab/rpg}{https://github.com/illidanlab/rpg}.}:

\textbf{EPG}: EPG is the stochastic listwise policy gradient (see~\eq{\ref{eq:lpgt:sample}}) 
incorporated with the proposed off-policy learning. More concretely, we apply 
trajectory reward shaping (TRS, Def~\ref{def:PI}) to all trajectories
encountered during exploration and train vanilla policy gradient using the off-policy samples.
This is equivalent to minimizing the cross-entropy loss (see Appendix~\eq{\ref{eq:lppg:cross}}) over the near-optimal trajectories.

\textbf{LPG}: LPG is the deterministic listwise policy gradient with the proposed off-policy 
learning. The only difference between EPG and LPG is that
LPG chooses action deterministically (see Appendix~\eq{\ref{eq:lpgt:argmax}}) during evaluation.

\textbf{RPG}: RPG explores the environment using a separate EPG agent in \textsc{Pong} and 
IQN in other games. Then RPG conducts supervised
learning by minimizing the hinge loss~\eq{\ref{eq:rankingv2}}. It is worth noting that
the exploration agent (EPG or IQN) can be replaced by any existing exploration method.
In our RPG implementation, we collect all trajectories with the trajectory reward no less
than the threshold $c$ without eliminating the duplicated trajectories 
and we empirically found it is a reasonable simplification. 

\noindent\textbf{Sample-efficiency.}
As the results shown in Figure~\ref{fig:pong}, our approach, RPG, significantly outperforms
the state-of-the-art baselines in terms of sample-efficiency at all tasks. Furthermore, 
RPG not only achieved the most sample-efficient results, but also reached the highest
final performance at \textsc{Robotank}, \textsc{DoubleDunk}, \textsc{Pitfall}, and \textsc{Pong},
comparing to any model-free state-of-the-art. 
In reinforcement learning, the stability of 
algorithm should be emphasized as an important issue. As we can see from the results,
the performance of baselines varies from task to task. There is no single baseline consistently
outperforms others. In contrast, due to the reduction from RL to supervised learning, RPG 
is consistently stable and effective across different environments. In addition to the 
stability and efficiency, RPG enjoys simplicity at the same time. In the environment \textsc{Pong},
it is surprising that RPG without any complicated exploration method largely surpassed
the sophisticated value-function based approaches. More details of hyperparameters  
are provided in the Appendix Section~\ref{sec:hyperparameters}.

\begin{figure}[t!]
\centering
\begin{tabular}{ccc}
\hspace{-2em}	\includegraphics[width=0.33\textwidth]{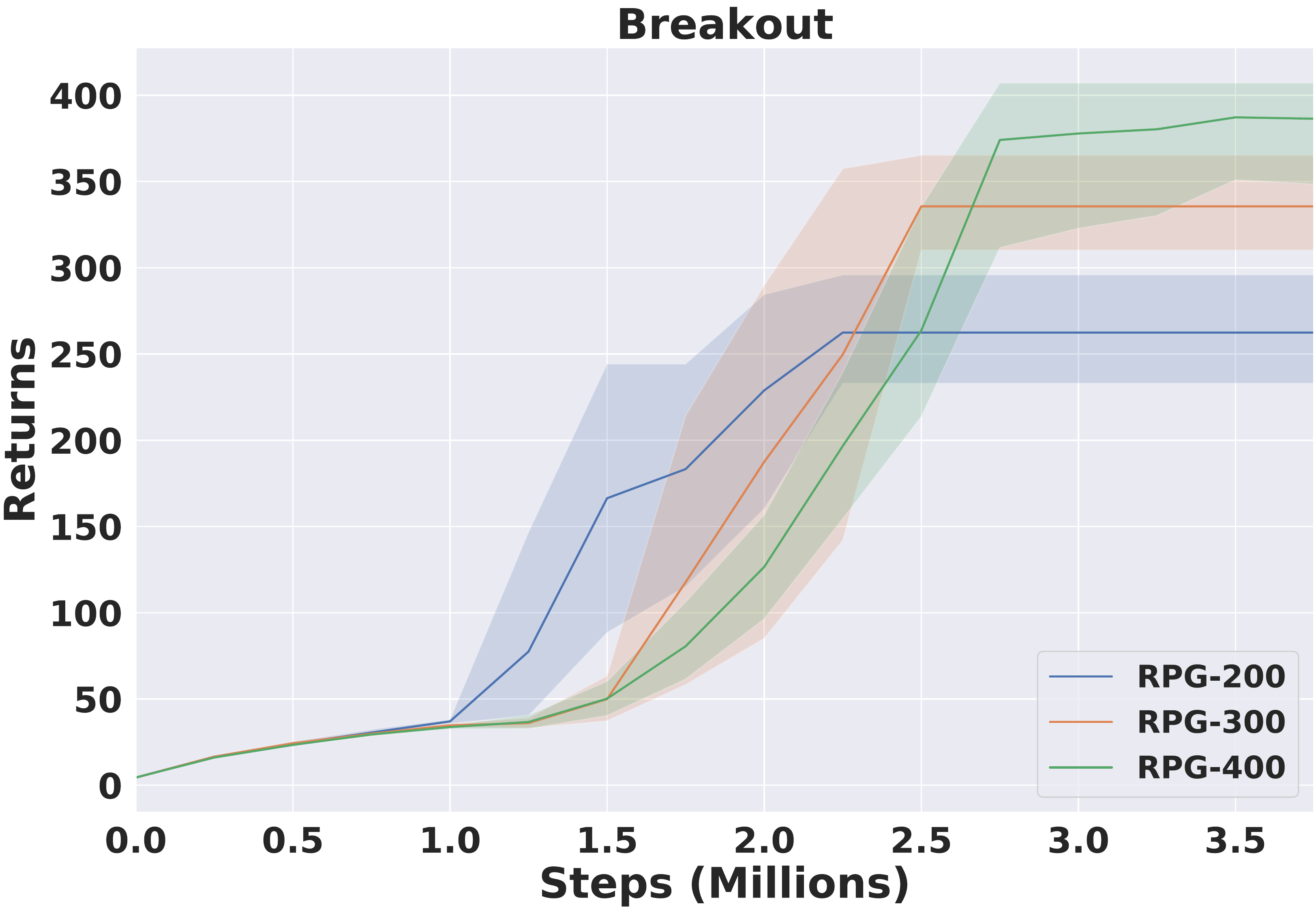} & 
	\includegraphics[width=0.35\textwidth]{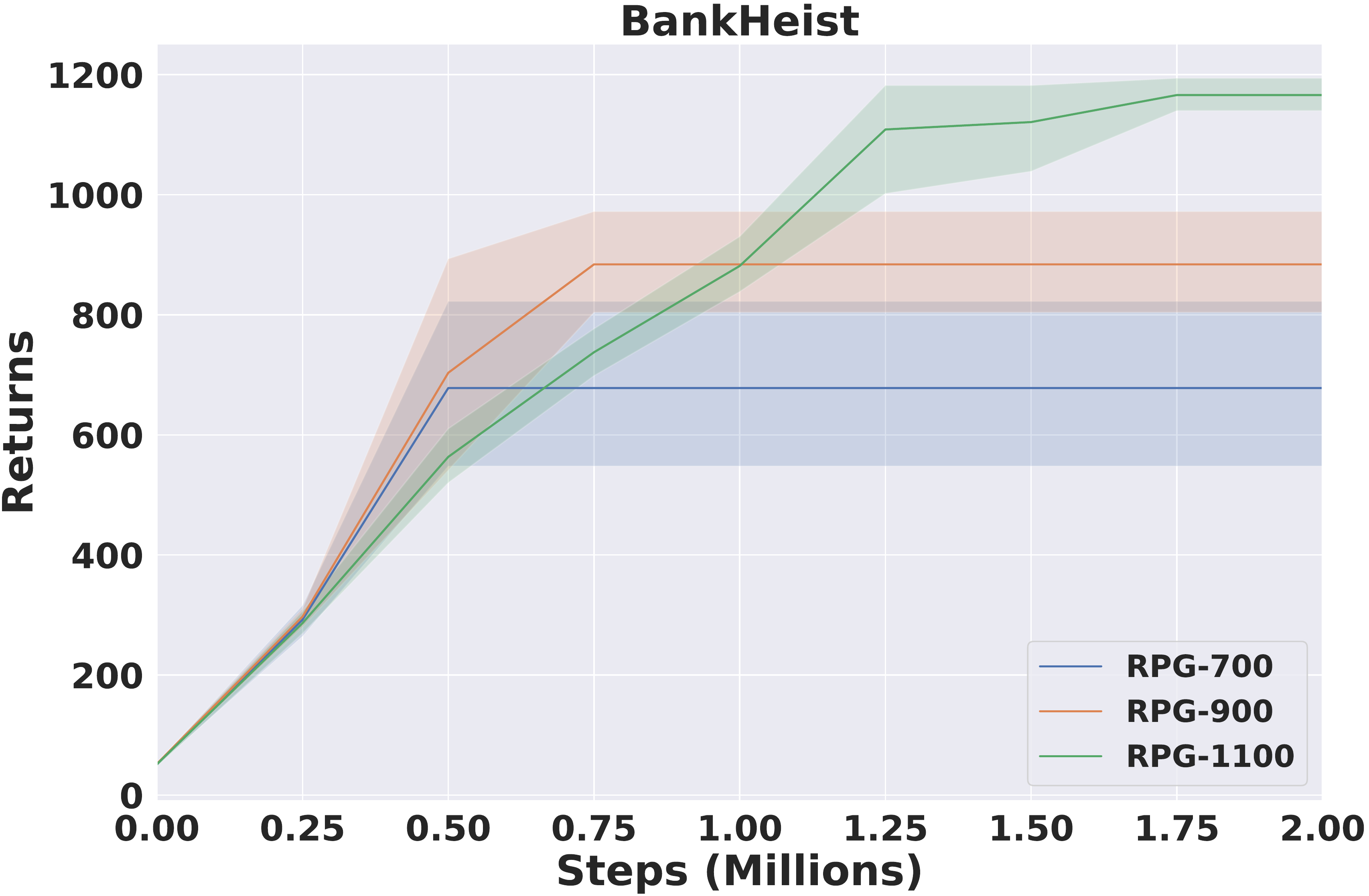} & 
	\includegraphics[width=0.33\textwidth]{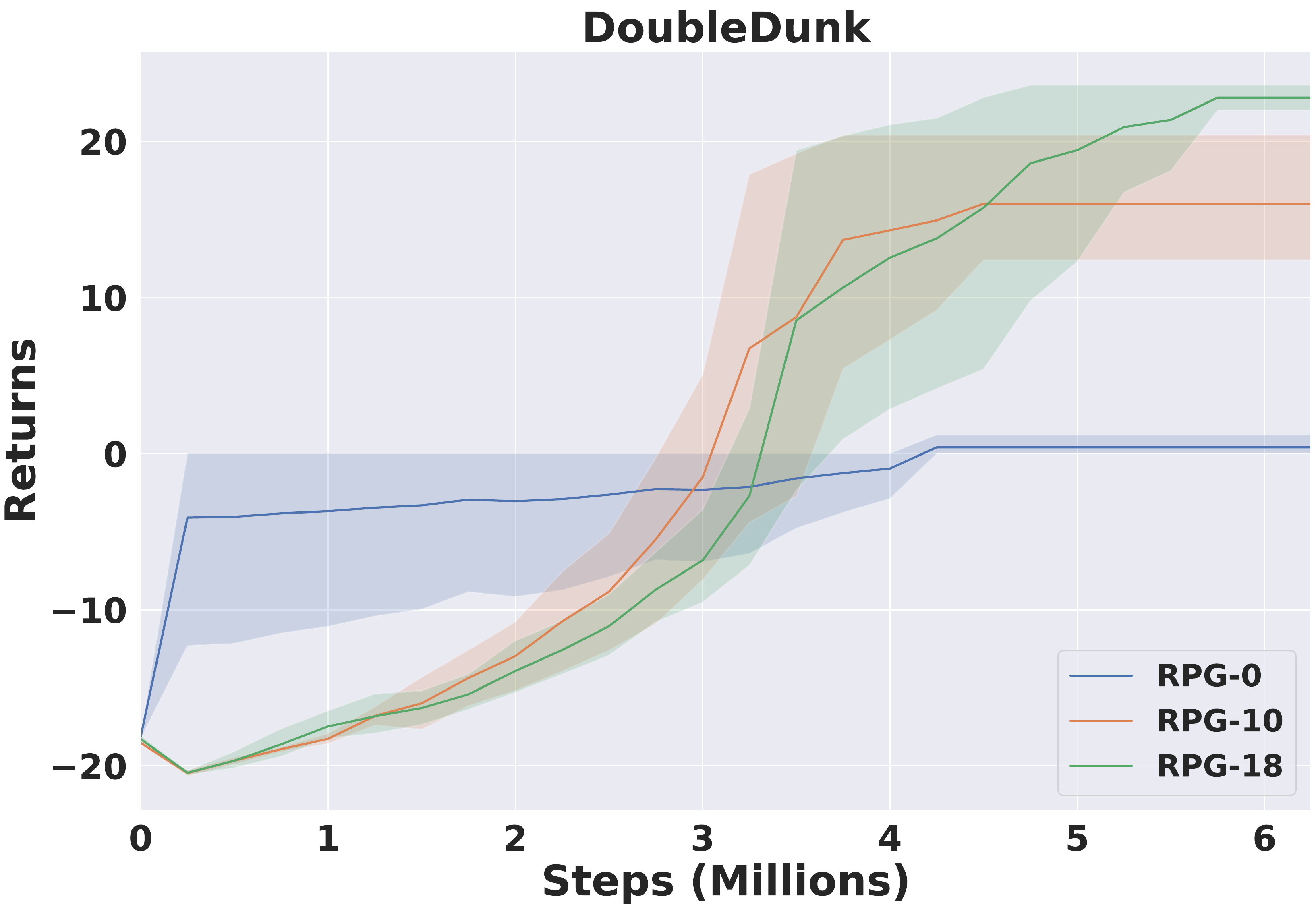} 
\end{tabular}
\vspace{-1em}
	\caption{The trade-off between sample efficiency and optimality on \textsc{DoubleDunk,BreakOut, BankHeist}. As the trajectory reward threshold ($c$)
	increase, more samples are needed for the learning to converge, while
	it leads to better final performance. We denote the value of $c$ by the numbers at the end of legends. }
	\label{fig:oetf}	
	\vspace{-2em}
\end{figure}
\subsection{Ablation Study}
\noindent\textbf{The effectiveness of pairwise ranking policy and off-policy learning as supervised learning.}
To get a better understanding of the underlying reasons that RPG is more
sample-efficient than DQN variants, we performed ablation studies in the
\textsc{Pong} environment by varying the combination of policy functions with
the proposed off-policy learning. The results of EPG, LPG, and RPG are shown
in the bottom right, Figure~\ref{fig:pong}. Recall that EPG and LPG use
listwise policy gradient (vanilla policy gradient using softmax as policy
function) to conduct exploration, the off-policy learning minimizes the cross-entropy
loss~\eq{\ref{eq:lppg:cross}}. In contrast, RPG shares the same exploration method as EPG and LPG 
while uses pairwise ranking policy~\eq{\ref{eq:rpg:policy}} in off-policy learning that
minimizes hinge loss~\eq{\ref{eq:rankingv2}}. We can see that RPG is more sample-efficient 
than EPG/LPG in learning deterministic optimal policy. We also compared the advanced on-policy method Proximal Policy Optimization (PPO)~\citep{schulman2017proximal} with EPG, LPG, and RPG. The proposed off-policy learning largely surpassed the best on-policy method.
Therefore, we conclude that off-policy as supervised learning
contributes to the sample-efficiency substantially, while the pairwise ranking policy can further 
accelerate the learning. In addition, we compare RPG to representative off-policy policy gradient approach:
ACER~\citep{wang2016sample}.
As the results shown, the proposed off-policy learning framework is more sample-efficient than the state-of-the-art off-policy policy gradient approaches.

\noindent\textbf{On the Trade-off between Sample-Efficiency and Optimality.}
Results in Figure~\ref{fig:oetf} show that there is a trade-off between sample efficiency
and optimality, which is controlled by the trajectory reward threshold $c$. 
Recall that $c$ determines how close is the learned UNOP to optimal policies. 
A higher value of $c$
leads to a less frequency of near-optimal trajectories being collected and 
and thus a lower sample efficiency, and however the algorithm is expected to converge to a strategy of better performance. 
We note that $c$ is the only parameter we tuned across all experiments.

\begin{figure}
\centering
\includegraphics[width=0.5\textwidth]{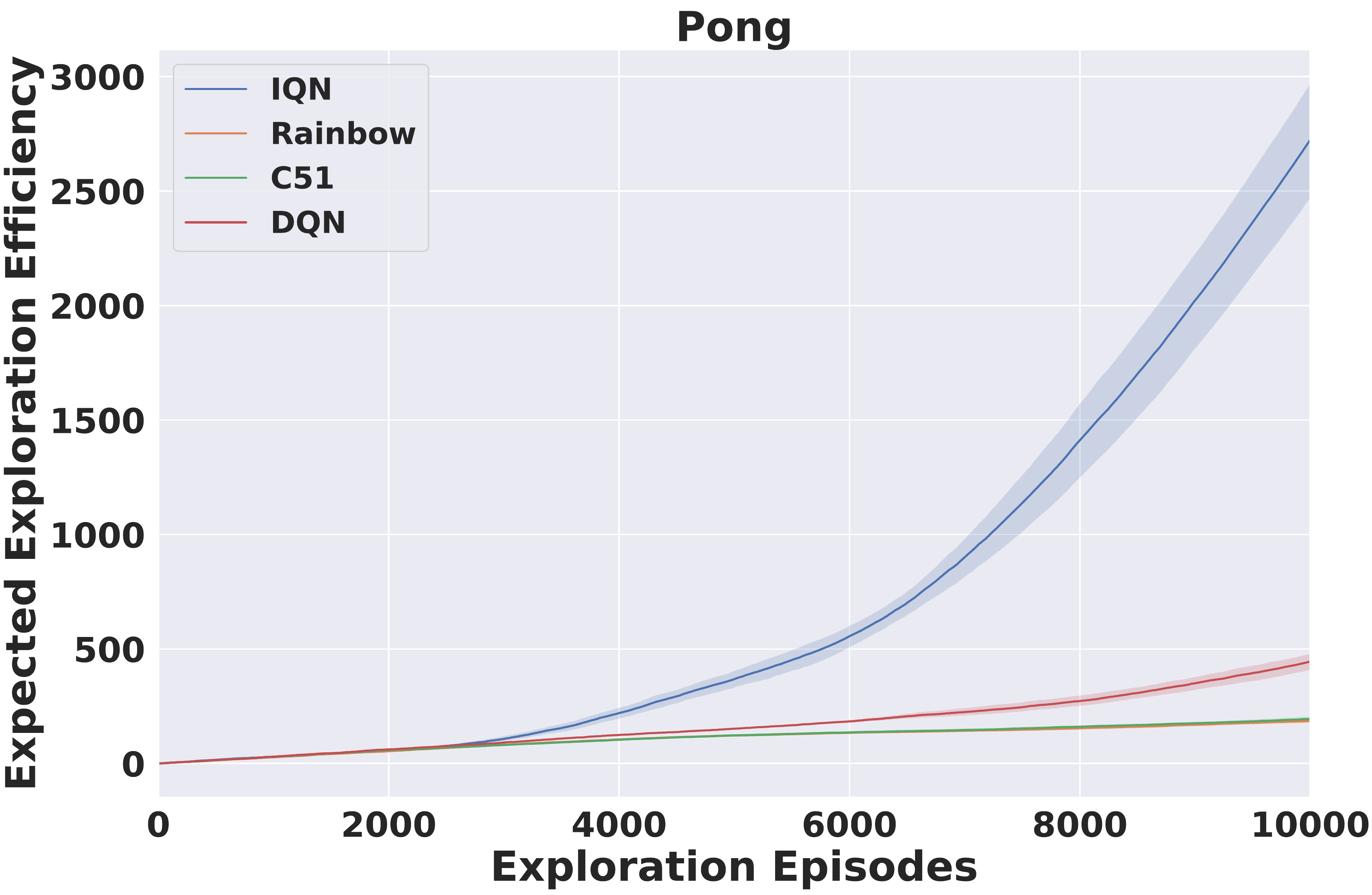}
	\caption{The expected exploration efficiency of state-of-the-art, the results are averaged over random seeds from 1 to 10. }
\label{fig:eee}	
\vspace{-1em}
\end{figure}
\noindent\textbf{Exploration Efficiency.}
We empirically evaluate the Expected Exploration Efficiency 
(Def~\ref{def:ee}) of the state-of-the-art on \textsc{Pong}. 
It is worth noting that the RL generalization performance is determined by both of 
learning efficiency and exploration efficiency. Therefore, higher exploration 
efficiency does not necessarily lead to more sample efficient algorithm due
to the learning inefficiency, 
as demonstrated by \textsc{RainBow} and \textsc{DQN} (see Figure~\ref{fig:eee}). 
Also, the Implicit Quantile 
achieves the best performance among baselines, since its 
exploration efficiency largely surpasses other baselines.

\vspace{-1em}
\section{Conclusions} 
\label{sec:discussions_and_conclusion}\balance

In this work, we introduced ranking policy gradient methods that,  for the
first time, approach the RL problem from a ranking perspective. Furthermore, 
towards the sample-efficient RL, we propose an off-policy learning framework,
which trains RL agents in a supervised learning manner and thus 
largely facilitates the learning efficiency. The off-policy learning framework uses generalized policy iteration for exploration and 
exploits the stableness of supervised learning for deriving policy, which accomplishes
the unbiasedness, variance reduction, off-policy learning, and sample efficiency at the same time.
Besides, we provide an alternative approach to analyze the sample complexity of
RL, and show that the sample complexity of RPG has no dependency on the state
space dimension. 
Last but not least, empirical results show that RPG achieves superior performance as compared to the state-of-the-art.

\small
\bibliography{ref.bib}
\onecolumn
\appendix


\begin{appendices}

\section{Discussion of Existing Efforts on Connecting Reinforcement
Learning to Supervised Learning.}
\label{sec:priworkRL2SL}

There are two main distinctions between supervised learning and reinforcement
learning. In supervised learning, the data distribution $\cD$ is static and training
samples are assumed to be sampled \emph{i.i.d.} from $\cD$. On the contrary, 
the data distribution is dynamic in reinforcement learning and the sampling
procedure is not independent. 
First, since the data distribution in RL is determined by both 
environment dynamics and the learning policy, and 
the policy keeps being updated during the learning process.
This updated policy results in dynamic data distribution in reinforcement learning. 
Second, policy learning depends on previously collected samples, which
in turn determines the sampling probability of incoming data. Therefore, 
the training samples we collected are not independently distributed. 
These intrinsic difficulties of reinforcement learning directly cause 
the sample-inefficient and unstable performance of current algorithms. 

On the other hand, most state-of-the-art reinforcement learning algorithms 
can be shown to have a supervised learning equivalent. To see this, 
recall that most reinforcement learning algorithms eventually acquire the policy
 either explicitly or implicitly, which is a mapping from a 
state to an action or a probability distribution over the action space. 
The use of such a mapping implies that ultimately
there exists a supervised learning equivalent to the original reinforcement learning problem, if optimal policies exist. 
The paradox is that it is almost impossible to construct this 
supervised learning equivalent on the fly, without knowing any optimal policy. 

Although the question of how to construct and apply proper supervision 
is still an open problem in the community, there are many existing efforts 
providing insightful approaches to reduce reinforcement learning into its 
supervised learning counterpart over the past several decades. 
Roughly, we can classify the existing efforts into the following 
categories:
\begin{itemize}
  \item \textit{Expectation-Maximization (EM)}: \cite{dayan1997using,peters2007reinforcement,kober2009policy,abdolmaleki2018maximum}, etc.
  \item \textit{Entropy-Regularized RL (ERL)}: \cite{o2016combining,oh2018self,haarnoja2018soft}, etc.
  \item \textit{Interactive Imitation Learning (IIL)}: \cite{daume2009search,syed2010reduction,ross2010efficient,ross2011reduction,sun2017deeply}, etc.
\end{itemize}

The early approaches in the EM track applied Jensen's inequality 
and approximation techniques to transform the reinforcement learning 
objective. Algorithms are then derived from the transformed objective, 
which resemble
the Expectation-Maximization procedure and provide policy improvement guarantee~\citep{dayan1997using}.
These approaches typically focus on a simplified
RL setting, such as assuming that the reward function is not
associated with the state \citep{dayan1997using}, approximating the goal to
maximize the expected immediate reward and the state distribution is assumed
to be fixed~\citep{peters2008reinforcement}. Later on in \cite{kober2009policy}, the authors extended the EM
framework from targeting immediate reward into episodic return. Recently, \cite{abdolmaleki2018maximum} used the EM-framework on a relative
entropy objective, which adds a parameter prior as regularization. 
It has been found that the estimation step using
\textit{Retrace}~\citep{munos2016safe} can be unstable even with a linear
function approximation~\citep{touati2017convergent}.  In general, the estimation
step in EM-based algorithms involves on-policy evaluation, which is
one challenge shared among policy gradient methods. On the 
other hand, off-policy learning usually leads to a much better 
sample efficiency, and is one main motivation that we want to reformulate 
RL into a supervised learning task. 

To achieve off-policy learning, PGQ~\citep{o2016combining} connected the
entropy-regularized policy gradient with Q-learning under the constraint of small
regularization. In the similar framework, Soft Actor-Critic
~\citep{haarnoja2018soft} was proposed to enable sample-efficient and 
faster convergence under the framework of entropy-regularized RL. 
It is able to converge to the optimal policy that optimizes the long-term reward
along with policy entropy. It is an efficient way to model the suboptimal behavior
and empirically it is able to learn a reasonable policy. Although recently the discrepancy
between the entropy-regularized objective and original long-term reward has been
discussed in~\citep{o2018variational,eysenbach2019if}, they focus on learning 
stochastic policy while the proposed framework is feasible for both learning deterministic
optimal policy (Corollary~\ref{cor:rppg}) and stochastic optimal policy (Corollary~\ref{cor:lppg}).
In~\citep{oh2018self}, 
this work shares similarity to our work in terms of 
the method we collecting the samples. They collect good samples based on the 
past experience and then conduct the imitation learning w.r.t those good samples. 
However, we differentiate at how do we look at the problem theoretically. This self-imitation
learning procedure was eventually connected to lower-bound-soft-Q-learning, which belongs to 
entropy-regularized reinforcement learning. We comment that there is a trade-off 
between sample-efficiency and modeling suboptimal behaviors. 
The more strict requirement we have on the samples
collected we have less chance to hit the samples while we are more close to imitating the optimal behavior.

\begin{table*}
\centering
\begin{tabular}{|c|c|c|c|c|c|} \hline
  \multicolumn{1}{|l|}{Methods} &  \multicolumn{1}{l|}{Objective} & \multicolumn{1}{l|}{Cont. Action} &  \multicolumn{1}{l|}{Optimality}& \multicolumn{1}{l|}{Off-Policy} & \multicolumn{1}{l|}{No Oracle} \\ \hline \hline 
   \multicolumn{1}{|p{1cm}|}{ EM} 
 & \multicolumn{1}{p{1cm}|}{ \checkmark } 
 & \multicolumn{1}{p{1cm}|}{ \checkmark }
 & \multicolumn{1}{p{1cm}|}{ \checkmark}
 & \multicolumn{1}{p{1cm}|}{ \xmark}
 & \multicolumn{1}{p{1cm}|}{ \checkmark} \\ \hline
   \multicolumn{1}{|p{1cm}|}{ ERL} 
 & \multicolumn{1}{p{1cm}|}{ \xmark } 
 & \multicolumn{1}{p{1cm}|}{ \checkmark}
 & \multicolumn{1}{p{1cm}|}{ \checkmark$^\dagger$}
 & \multicolumn{1}{p{1cm}|}{ \checkmark}
 & \multicolumn{1}{p{1cm}|}{ \checkmark} \\ \hline
   \multicolumn{1}{|p{1cm}|}{ IIL} 
  & \multicolumn{1}{p{1cm}|}{ \checkmark } 
 & \multicolumn{1}{p{1cm}|}{ \checkmark}
 & \multicolumn{1}{p{1cm}|}{ \checkmark}
 & \multicolumn{1}{p{1cm}|}{ \checkmark}
 & \multicolumn{1}{p{1cm}|}{ \xmark} \\ \hline
%
    \multicolumn{1}{|p{1cm}|}{ RPG} 
 & \multicolumn{1}{p{1cm}|}{ \checkmark } 
 & \multicolumn{1}{p{1cm}|}{ \xmark}
 & \multicolumn{1}{p{1cm}|}{ \checkmark}
 & \multicolumn{1}{p{1cm}|}{ \checkmark}
 & \multicolumn{1}{p{1cm}|}{ \checkmark} \\ \hline
\end{tabular}
\caption{A comparison of studies reducing RL to SL. 
The \textit{Objective} column denotes whether
the goal is to maximize long-term reward.  
The \textit{Cont. Action} column denotes whether the 
method is applicable to both continuous and discrete action spaces. 
The \textit{Optimality}
denotes whether the algorithms can model the optimal policy. 
\checkmark$^\dagger$ denotes the optimality achieved by ERL 
is w.r.t. the entropy regularize objective instead of 
the original objective on return.
The \textit{Off-Policy} column denotes
if the algorithms enable off-policy learning. 
The \textit{No Oracle} column 
denotes if the algorithms need to access to a certain type of oracle (expert policy or expert 
demonstrations).  }
\label{tb:rl2sl}
\end{table*}  

From the track of interactive imitation learning, early 
efforts such as~\citep{ross2010efficient,ross2011reduction} pointed out that
the main discrepancy between imitation learning and reinforcement learning is
the violation of \emph{i.i.d.} assumption. 
\textsc{SMILe}~\citep{ross2010efficient} and
\textsc{DAgger}~\citep{ross2011reduction} are proposed to overcome the distribution mismatch.  
Theorem 2.1 in \cite{ross2010efficient} quantified the performance degradation
from the expert considering that the learned policy fails to imitate the expert with a certain probability.
The theorem seems to resemble the long-term performance theorem (Thm.~\ref{th:ltpt}) in this paper. However, it
studied the scenario that the learning policy is trained through a state
distribution induced by the expert, instead of state-action distribution as considered in Theorem~\ref{th:ltpt}. 
As such, Theorem 2.1 in \cite{ross2010efficient} may be more applicable to the
situation where an interactive procedure is needed, such as querying the
expert during the training process. On the contrary, the proposed work focuses on directly
applying supervised learning without having access to the expert to
label the data. The optimal state-action pairs are collected during
exploration and conducting supervised learning on the replay buffer will
provide a performance guarantee in terms of long-term expected reward.
Concurrently, a resemble of Theorem 2.1 in~\citep{ross2010efficient} is Theorem
1 in~\citep{syed2010reduction}, where the authors reduced the apprenticeship 
learning to classification, under the assumption that the apprentice policy is
deterministic and the misclassification rate is bounded at all time steps.
In this work, we show that it is possible to circumvent such a strong assumption
and reduce RL to its SL. Furthermore, our theoretical framework also 
leads to an alternative analysis of sample-complexity.
Later on 
\textsc{AggreVaTe}~\citep{ross2014reinforcement} was proposed to incorporate
the information of action costs to facilitate imitation learning, and its
differentiable version \textsc{AggreVaTeD}~\citep{sun2017deeply} was
developed in succession and achieved impressive empirical results. Recently, 
hinge loss was introduced to regular $Q$-learning as a pre-training step
for learning from demonstration~\citep{hester2018deep}, or as a surrogate loss
for imitating optimal trajectories~\citep{osa2018algorithmic}. In this work,
we show that hinge loss constructs a new type of policy gradient method and
can be used to learn optimal policy directly. 

In conclusion, our method approaches the problem of reducing RL to SL from a
unique perspective that is different from all prior work. With our reformulation
from RL to SL, the samples collected in the replay buffer satisfy 
the \emph{i.i.d.}
assumption, since the state-action pairs are now sampled from the data
distribution of UNOP. A multi-aspect comparison between the proposed method 
and relevant prior studies is summarized in Table~\ref{tb:rl2sl}.



\section{Ranking Policy Gradient Theorem}
\label{sec:th:rpgt}

The Ranking Policy Gradient Theorem (Theorem~\ref{th:rpgt}) formulates the
optimization of long-term reward using a ranking objective. The proof below
illustrates the formulation process.  

\begin{proof}
The following proof is based on direct policy differentiation~\citep{peters2008reinforcement,williams1992simple}.
For a concise presentation, the subscript $t$ for action value $\lambda_i, \lambda_j$, and $p_{ij}$ is omitted.
\begin{align}
 \grad_{\theta} J({\theta}) =& \grad_{\theta} \sum_{\tau}\nolimits p_{\theta}(\tau) r(\tau) \label{eq:rlpgobj} \\
=& \sum_\tau\nolimits p_{\theta}(\tau) \grad_{\theta}\log p_{\theta}(\tau)r(\tau) \nonumber \\
=& \sum_\tau\nolimits p_{\theta}(\tau) \grad_{\theta}\log \left(p(s_0)\Pi_{t=1}^{T}\pi_{\theta}(a_t|s_t)p(s_{t+1}|s_t, a_t)\right)r(\tau)  \nonumber \\
=& \sum_\tau\nolimits p_{\theta}(\tau)  \sum_{t=1}^{T}\nolimits\grad_{\theta}\log\pi_{\theta}(a_t|s_t)r(\tau)  \nonumber \\
=& \vE_{\tau \sim \pi_\theta}\left[\sum_{t=1}^{T}\nolimits \grad_{\theta} \log\pi_{\theta}(a_t|s_t) r(\tau)\right ]\nonumber\\
=& \vE_{\tau \sim \pi_\theta}\left[\sum_{t=1}^{T}\nolimits\grad_{\theta} \log\left(\prod\nolimits_{j=1, j\neq i}^{m} p_{ij}\right) r(\tau)\right]\nonumber \\
=& \vE_{\tau \sim \pi_\theta}\left[\sum_{t=1}^{T}\nolimits\grad_{\theta} \sum_{j=1, j\neq i}^{m}\nolimits \log\left(\frac{e^{\lambda_{ij}}}{1+ e^{\lambda_{ij}}} \right)r(\tau)\right]\nonumber\\
=& \vE_{\tau \sim \pi_\theta}\left[\sum_{t=1}^{T}\nolimits\grad_{\theta} \sum_{j=1, j\neq i}^{m}\nolimits \log\left(\frac{1}{1+ e^{\lambda_{ji}}} \right)r(\tau)\right]\label{eq:rpg_ori} \\
\approx& \vE_{\tau \sim \pi_\theta}\left[\sum_{t=1}^{T}\nolimits\grad_{\theta} \left(\sum_{j=1, j\neq i}^{m}\nolimits (\lambda_i-\lambda_j)/2\right)r(\tau)\right],
\label{eq:rpg_first}
\end{align}
where the trajectory is a series of state-action pairs from $t=1,...,T$, $i.e.\tau = s_1,a_1,s_2,a_2,...,s_T$.
From \eq{\ref{eq:rpg_ori}} to \eq{\ref{eq:rpg_first}}, we use the 
first-order Taylor expansion of $\log(1+e^x)|_{x=0} = \log2 + \frac{1}{2}x + O(x^2)$ to further simplify the ranking policy gradient. 
\label{proof:rpg}
\end{proof}
\vspace{-1em}

\subsection{Probability Distribution in Ranking Policy Gradient}
\label{subsec:rpgvalid}

In this section, we discuss the output property of the pairwise ranking policy. 
We show in Corollary
\ref{corollary:rpg:pd} that the pairwise ranking policy gives a valid
probability distribution when the dimension of  the action space $m=2$. For
cases when $m>2$ and the range of $Q$-value satisfies Condition
\ref{cond:qvalue}, we show in Corollary \ref{corollary:valid:prob} how to
construct a valid probability distribution. 

\begin{corollary}
The pairwise ranking policy as shown in~\eq{\ref{eq:rpg:policy}}
constructs a probability distribution over the set of actions 
when the action space $m$ is equal to $2$, given any action values 
$\lambda_i, i=1,2$. For the cases with $m>2$, this conclusion does not hold in general.
\label{corollary:rpg:pd}
\end{corollary}
It is easy to verify that 
$\pi(a_i|s)>0$, $\sum_{i=1}^2 \pi(a_i|s)=1$ holds and 
the same conclusion cannot be applied to $m>2$ by constructing counterexamples. However, 
we can introduce a dummy action $a'$ to form a probability distribution for RPG. 
During policy learning, the algorithm increases the probability of best actions
and the probability of dummy action decreases. Ideally, if RPG converges to an optimal 
deterministic policy, the probability of taking best action is equal to 1 and $\pi(a'|s) = 0$. 
Similarly, we can introduce a dummy trajectory $\tau'$ with the trajectory reward $r(\tau') = 0$
and $p_{\theta}(\tau')= 1 - \sum_\tau p_{\theta}(\tau)$. 
The trajectory probability forms a probability distribution
since $\sum_\tau p_{\theta}(\tau) + p_{\theta}(\tau')= 1$ and $p_{\theta}(\tau) \geq 0 \ \forall \tau$
and $p_{\theta}(\tau')\geq 0$. 
The proof of a valid trajectory probability is similar to
the following proof on $\pi(a|s)$ to be a valid probability distribution with a dummy action. 
Its practical influence is negligible since our goal is to increase the probability
of (near)-optimal trajectories. To present in a clear way, we avoid mentioning 
dummy trajectory $\tau'$ in Proof~\ref{proof:rpg} while it can be seamlessly included. 

\begin{condition}[The range of action-value]
	We restrict the range of action-values in RPG so that it satisfies 
	$\lambda_m  \geq \ln (m^{\frac{1}{m-1}} - 1)$, where $\lambda_m = \min_{i,j} \lambda_{ji}$ and $m$ is the dimension of the action space. 
	\label{cond:qvalue}
\end{condition}
This condition can be easily satisfied since in RPG we only focus
on the relative relationship of $\lambda$ and we can constrain the range of 
action-values so that $\lambda_m$ satisfies the condition~\ref{cond:qvalue}. Furthermore,
since we can see that $m^{\frac{1}{m-1}} > 1$ is decreasing w.r.t to action dimension $m$.
The larger the action dimension, the less constraint we have on the action values.

\begin{corollary}
Given Condition~\ref{cond:qvalue}, we introduce a dummy action $a'$ and set
$\pi(a=a'|s) = 1 - \sum_{i} \pi(a=a_i|s)$, which constructs a valid probability 
distribution $\left(\pi(a|s)\right)$ over the action space $\cA \cup a'$.
\label{corollary:valid:prob}
\end{corollary}
\begin{proof}
Since we have $\pi(a=a_i|s) > 0\   \forall i = 1,\dots,m$ and $\sum_i \pi(a=a_i|s) + \pi(a=a'|s) = 1 $. 
To prove that this is a valid probability distribution, we only need to show that $\pi(a=a'|s)\geq 0,\  \forall m \geq 2$, i.e.  $\sum_{i} \pi(a=a_i|s) \leq 1, \  \forall m \geq 2$. 
Let $\lambda_m = \min_{i,j} \lambda_{ji}$, 
	\begin{align*}
				& \sum_{i}\nolimits\pi(a=a_i|s)\nonumber \\
				=&  \sum_i\nolimits \prod\nolimits_{j=1, j\neq i}^{m}\ p_{ij}\nonumber\\
				=& \sum_i\nolimits \prod\nolimits_{j=1, j\neq i}^{m} \frac{1}{1+ e^{\lambda_{ji}}} \nonumber \\
				\leq & \sum_i\nolimits \prod\nolimits_{j=1, j\neq i}^{m} \frac{1}{1+ e^{\lambda_m} }\\ 
				= & m \left(\frac{1}{1+ e^{\lambda_m} }\right)^{m-1} \leq   1  \qquad \text{ (Condition~\ref{cond:qvalue})}. 
	\end{align*}
This thus concludes the proof. 	
\end{proof}



\section{Condition of Preserving Optimality}
\label{sec:th:optimal}

The following condition describes what types of MDPs are directly
applicable to the trajectory reward shaping (TRS, Def~\ref{def:PI}): 

\begin{condition}[Initial States]
The (near)-optimal trajectories will cover all initial states of MDP. i.e. $\{s(\tau,1) |\ \forall \tau \in \cT \} = \{s(\tau,1)|\ \forall \tau\} $, where $\cT = \{\tau | w(\tau)=1\}=\{\tau | r(\tau)\geq c\}$.
\label{cond:is}
\end{condition}
The MDPs satisfying this 
condition cover a wide range of tasks such as Dialogue System~\citep{li2017end},
Go~\citep{silver2017mastering}, video games~\citep{bellemare2013arcade} 
and all MDPs with only one initial state. If we want to preserve the optimality by TRS,
the optimal trajectories of a MDP need to cover all initial states or equivalently,
all initial states must lead to at least one optimal trajectory. Similarly, the near-optimality
is preserved for all MDPs that its near-optimal trajectories cover all initial states. 

\begin{figure*}
\centering
\begin{tabular}{cc}
\includegraphics[width=0.45\textwidth]{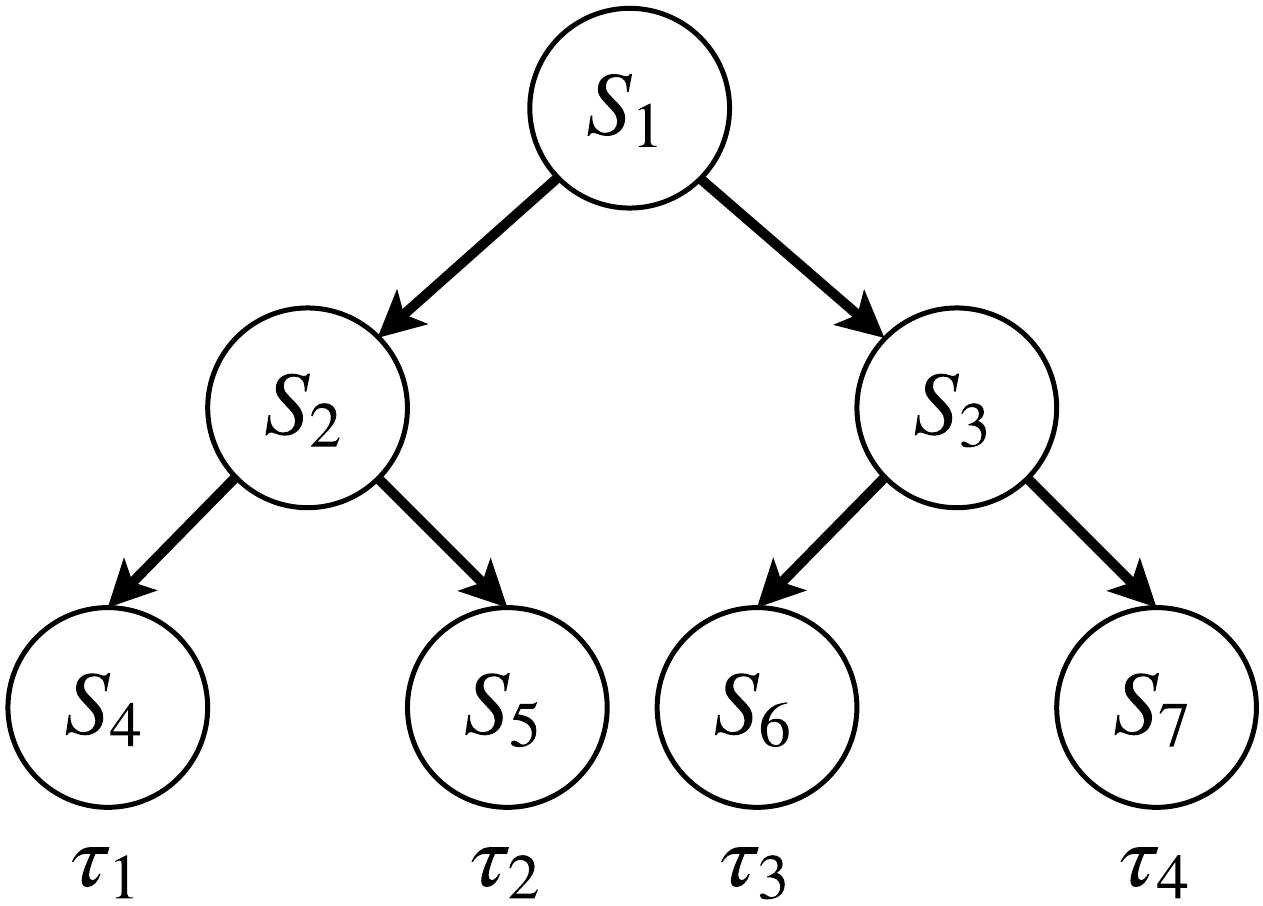} &
	\includegraphics[width=0.45\textwidth]{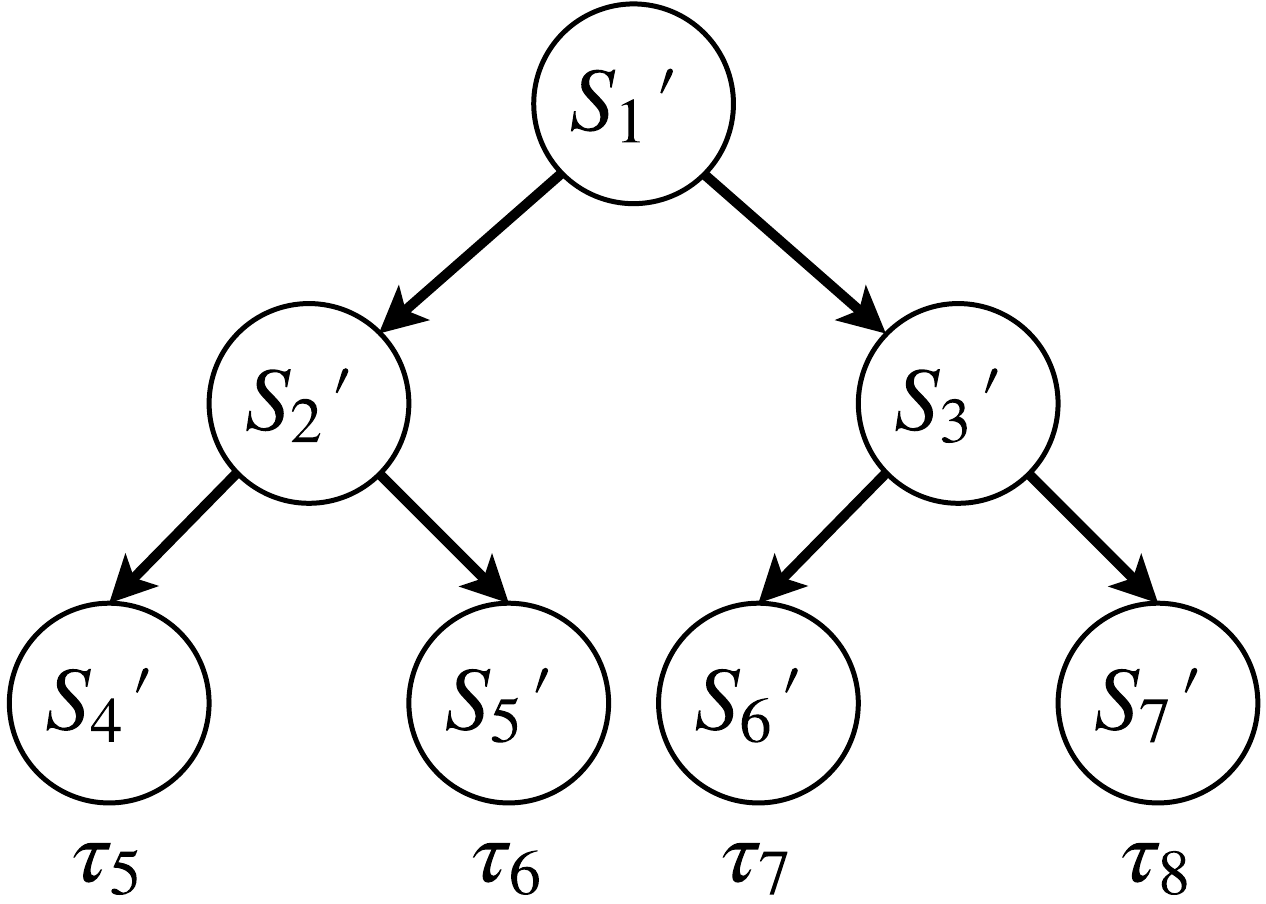} 
\end{tabular}
\caption{
The binary tree structure MDP with two initial states ($\cS_1 = \{s_1,s_1'\}$), similar as discussed
in~\citep{sun2017deeply}. Each path from a root to a leaf node denotes one possible 
trajectory in the MDP.}
\label{fig:mdps-2}
\end{figure*}
Theoretically, it is possible to transfer more general MDPs to satisfy Condition~\ref{cond:is} 
and preserve the optimality with potential-based reward shaping~\citep{ng1999policy}. 
More concretely, consider the deterministic binary tree MDP ($\cM_1$) with the set of initial states
$\cS_1 = \{s_1,s_1'\}$ as defined in Figure~\ref{fig:mdps-2}.
There are eight possible trajectories in $\cM_1$. Let $r(\tau_1) = 10 = R_{\max}, r(\tau_8)=3, \ r(\tau_i) = 2, \ \forall i = 2, \dots, 7$. Therefore, this MDP does not satisfy Condition~\ref{cond:is}. We can compensate the trajectory reward of the best trajectory starting from $s_1'$
to the $R_{\max}$ by shaping the reward with the potential-based function $\phi(s_7') = 7$ and  
$\phi(s) = 0, \forall s \neq s_7'$. This reward shaping requires more prior knowledge, which may not be
feasible in practice. A more realistic method is to design a dynamic trajectory reward shaping approach.
In the beginning, we set $c(s) = \min_{s \in \cS_1} r(\tau|s(\tau, 1) = s), \forall s \in \cS_1$. 
Take $\cM_1$ as an example, $c(s) = 3,\ \forall s \in \cS_1$. 
During the exploration stage, we track the current best trajectory of each initial state 
and update $c(s)$ with its trajectory reward. 

Nevertheless, if the Condition~\ref{cond:is} is not satisfied, 
we need more sophisticated prior knowledge other than a predefined
trajectory reward threshold $c$ to construct the replay buffer (training dataset of UNOP).
The practical implementation of trajectory reward shaping and rigorously theoretical study 
for general MDPs are beyond the scope of this work.


\section{Proof of Long-term Performance Theorem~\ref{th:ltpt}}
\label{sec:th:ltpt}

\begin{lemma}
Given a specific trajectory $\tau$, the log-likelihood of state-action pairs over horizon $T$ is equal to 
the weighted sum over the entire state-action space, i.e.:
\begin{align}
\frac{1}{T}\sum_{t=1}^{T}\nolimits\log\pi_{\theta}(a_t|s_t) = \sum_{s, a}\nolimits p(s, a|\tau)\log\pi_{\theta}(a|s), \nonumber
\end{align}
where the sum in the right hand side is the summation over all possible state-action pairs. It is worth
noting that $p(s, a|\tau)$ is not related to any policy parameters. It is the probability of
a specific state-action pair $(s,a)$ in a specific trajectory $\tau$.
\label{lemma:tlp}
\end{lemma}
\begin{proof}
	 Given a trajectory $\tau=\{ (s(\tau, 1), a(\tau, 1)), \dots, (s(\tau, T), a(\tau, T))\}= \{(s_1, a_1), \dots, (s_T, a_T)\}$, denote the unique state-action pairs in this trajectory as $U(\tau) = \{(s_i, a_i)\}_{i=1}^n$, where $n$ is 
	 the number of unique state-action pairs in $\tau$ and $n \leq T$.
	 The number of occurrences of a state-action pair $(s_i, a_i)$ in the trajectory $\tau$ is denoted as $|(s_i, a_i)|$. Then we have the following: 
	 \begin{align}
	 	& \frac{1}{T}\sum_{t=1}^{T}\nolimits\log\pi_{\theta}(a_t|s_t) \nonumber \\
	  = & \sum_{i=1}^n\nolimits  \frac{|(s_i, a_i)|}{T}\log\pi_{\theta}(a_i|s_i) \nonumber \\
	  = & \sum_{i=1}^n\nolimits  p(s_i, a_i|\tau)\log\pi_{\theta}(a_i|s_i) \nonumber \\	
	  = & \sum_{(s, a) \in U(\tau)} \nolimits p(s, a|\tau)\log\pi_{\theta}(a|s) 	\label{lemma:tlp:2}\\
	  = & \sum_{(s, a) \in U(\tau)} \nolimits p(s, a|\tau)\log\pi_{\theta}(a|s)+ \sum_{(s, a) \notin U(\tau)} \nolimits p(s, a|\tau)\log\pi_{\theta}(a|s)\label{lemma:tlp:3}	\\
	  = & \sum_{(s, a)}\nolimits  p(s, a|\tau)\log\pi_{\theta}(a|s)	\nonumber
	 \end{align}
From \eq{\ref{lemma:tlp:2}} to \eq{\ref{lemma:tlp:3}} we used the fact: 
\begin{align*}
\sum\nolimits_{(s, a) \in U(\tau)}  p(s, a|\tau) = \sum\nolimits_{i=1}^n p(s_i, a_i|\tau) = \sum\nolimits_{i=1}^n\frac{|(s_i, a_i)|}{T} = 1,  
\end{align*}
and therefore we have $p(s, a|\tau) = 0, \ \forall (s, a) \notin U(\tau)$.
This thus completes the proof. 
\end{proof}

Now we are ready to prove the Theorem~\ref{th:ltpt}:
\begin{proof}
The following proof holds for an arbitrary subset of trajectories $\cT$ determined by the
threshold $c$ in Def~\ref{def:uop}. The $\pi_*$ is associated with $c$ and this subset of 
trajectories. We present the following lower bound of the expected long-term performance: 
\begin{align}
& \argmax_{\theta} \sum_{\tau} \nolimits p_{\theta}(\tau) w(\tau) \nonumber\\
& \because w(\tau) = 0,  \text{if} \ \tau \notin \cT \nonumber \\
=&\argmax_{\theta} \frac{1}{|\cT|}\sum_{\tau\in \cT}\nolimits  p_{\theta}(\tau)w(\tau) \nonumber \\
& \text{use Lemma~\ref{lemma:pp} }\because p_{\theta}(\tau)>0 \text{ and } w(\tau) > 0, \therefore \sum_{\tau\in \cT}\nolimits  p_{\theta}(\tau)w(\tau)>0 \nonumber\\
=&\argmax_{\theta} \log \left(\frac{1}{|\cT|}\sum_{\tau\in \cT} \nolimits p_{\theta}(\tau)w(\tau) \right) \nonumber \\
& \because \log\left(\sum_{i=1}^n\nolimits  x_i/n \right) \geq \sum_{i=1}^n\nolimits \log( x_i)/n, \forall i, x_i>0, \text{we have: } \nonumber\\
& \log \left(\frac{1}{|\cT|}\sum_{\tau\in \cT}\nolimits  p_{\theta}(\tau)w(\tau) \right) \geq  \sum_{\tau\in \cT} \nolimits \frac{1}{|\cT|}\log p_{\theta}(\tau)w(\tau), \nonumber
\end{align}
where the lower bound holds when $p_{\theta}(\tau)w(\tau) = \frac{1}{|\cT|}, \forall \tau \in \cT$. 
To this end, we maximize the lower bound of the expected long-term performance: 
\begin{align}
 & \argmax_{\theta} \sum_{\tau\in \cT}\nolimits  \frac{1}{|\cT|}\log p_{\theta}(\tau)w(\tau) \nonumber\\
=&\argmax_{\theta} \sum_{\tau\in\cT} \nolimits \log(p(s_1)\prod\nolimits_{t=1}^{T}(\pi_{\theta}(a_t|s_t)p(s_{t+1}|s_{t}, a_{t})) w(\tau) ) \nonumber \\
 =&\argmax_{\theta} \sum_{\tau\in\cT} \nolimits \log\left(p(s_1)\prod\nolimits_{t=1}^{T}\pi_{\theta}(a_t|s_t)\prod\nolimits_{t=1}^{T}p(s_{t+1}|s_{t}, a_{t}) w(\tau)\right) \nonumber\\
=& \argmax_{\theta} \sum_{\tau\in\cT} \nolimits \left( \log p(s_1) + \sum\nolimits_{t=1}^{T}\log p(s_{t+1}|s_{t}, a_{t})+ \sum_{t=1}^{T}\nolimits \log\pi_{\theta}(a_t|s_t) + \log w(\tau) \right) \label{eq:constantw}\\
& \text{The above shows that } w(\tau) \text{ can be set as an arbitrary positive constant } \nonumber\\
=& \argmax_{\theta} \frac{1}{|\cT|} \sum_{\tau\in\cT} \nolimits \sum_{t=1}^{T}\nolimits \log\prod\nolimits_{\theta}(a_t|s_t) \nonumber\\
=& \argmax_{\theta} \frac{1}{|\cT|T} \sum_{\tau\in\cT} \nolimits \sum_{t=1}^{T}\nolimits \log\prod\nolimits_{\theta}(a_t|s_t)  \label{eq:ltpt:sumtraj}\\
=& \argmax_{\theta} \frac{1}{|\cT|} \sum_{\tau\in\cT}\nolimits \frac{1}{T} \sum_{t=1}^{T}\nolimits \log\pi_{\theta}(a_t|s_t) \quad\text{ (the existence of UNOP in Assumption~\ref{ass:uot})  } \nonumber \\
=&\argmax_{\theta}  \sum_{\tau \in\cT}\nolimits  p_{\pi_*}(\tau) \frac{1}{T} \left(\sum_{t=1}^{T} \nolimits \log\pi_{\theta}(a_t|s_t) \right) \nonumber\\
& \text{where } \pi_* \text{ is a UNOP (Def~\ref{def:uop})} \Rightarrow p_{\pi_*}(\tau) = 0 \quad \forall \tau \notin \cT  \label{eq:ltpt:48}\\
& \text{\eq{\ref{eq:ltpt:48}} can be established based on }  \sum_{\tau \in\cT}\nolimits  p_{\pi_*}(\tau) = \sum_{\tau \in\cT} \nolimits 1/|\cT| = 1\nonumber\\
=&\argmax_{\theta}  \sum_{\tau }\nolimits  p_{\pi_*}(\tau) \frac{1}{T} \left(\sum\nolimits_{t=1}^{T}\log\pi_{\theta}(a_t|s_t) \right) \qquad \text{(Lemma~\ref{lemma:tlp})}\nonumber\\
=&\argmax_{\theta}  \sum_{\tau } \nolimits  p_{\pi_*}(\tau)    \sum_{s, a}\nolimits  p(s, a|\tau)\log\pi_{\theta}(a|s) \nonumber \\
& \text{The 2nd sum is over all possible state-action pairs. $(s, a)$ represents a specific state-action pair.}\nonumber\\
=&\argmax_{\theta}  \sum_{\tau }\nolimits \sum_{s,a} \nolimits  p_{\pi_*}(\tau)  p(s, a|\tau) \log\pi_{\theta}(a|s) \nonumber \\
=&\argmax_{\theta}  \sum_{s,a}\nolimits  \sum_{\tau}\nolimits  p_{\pi_*}(\tau)  p(s, a|\tau) \log\pi_{\theta}(a|s) \nonumber \\
=&\argmax_{\theta}  \sum_{s,a} \nolimits p_{\pi_*}(s, a) \log\pi_{\theta}(a|s).\label{eq:ltpt:slapp}
\end{align}
In this proof we use $s_t=s(\tau, t)$ and $a_t=a(\tau, t)$ as abbreviations, which denote the
$t$-th state and action in the trajectory $\tau$, respectively.
$|\cT|$ denotes the number of trajectories in $\cT$. We also use the definition
of $w(\tau)$ to only focus on near-optimal trajectories. We set $w(\tau) = 1$ for simplicity but it
will not affect the conclusion if set to other constants. 

\textbf{Optimality:}
Furthermore, the optimal solution
for the objective function~\eq{\ref{eq:ltpt:slapp}} is a uniformly (near)-optimal policy $\pi_*$. 
\begin{align}
	 & \argmax_{\theta}\sum_{s,a}\nolimits  p_{\pi_*}(s, a) \log\pi_{\theta}(a|s) \nonumber \\
	=& \argmax_{\theta}\sum_{s} \nolimits p_{\pi_*}(s) \sum_a \nolimits \pi_*(a|s) \log\pi_{\theta}(a|s) \nonumber\\
	=& \argmax_{\theta}\sum_{s} \nolimits p_{\pi_*}(s) \sum_a \nolimits \pi_*(a|s) \log\pi_{\theta}(a|s) - \sum_{s} \nolimits p_{\pi_*}(s) \sum_a \nolimits  \log \pi_*(a|s) \nonumber\\
	= & \argmax_{\theta}\sum_{s} \nolimits p_{\pi_*}(s) \sum_a \nolimits \nolimits \pi_*(a|s) \log\frac{\pi_{\theta}(a|s)}{\pi_*(a|s)} \nonumber\\
	= & \argmax_{\theta}\sum_{s} \nolimits p_{\pi_*}(s) \sum_a \nolimits - KL(\pi_*(a|s)||\pi_{\theta}(a|s))
	 = \pi_* \nonumber
\end{align}

Therefore,
the optimal solution of ~\eq{\ref{eq:ltpt:slapp}} is also the (near)-optimal solution for the original RL
problem since 
$\sum_{\tau} p_{\pi_*}(\tau)r(\tau) = \sum_{\tau\in\cT} \frac{1}{|\cT|} r(\tau) \geq c = R_{\max}-\epsilon $. The optimal solution is obtained when we set $c = R_{\max}$. 
\label{pf:ltpt}
\end{proof}

\begin{lemma}
Given any optimal policy $\pi$ of MDP satisfying Condition~\ref{cond:is},  $ \forall \tau \notin \cT$, we have $p_{\pi}(\tau) = 0 $
, where $\cT$ denotes the set of all possible optimal trajectories in this lemma.
If $\exists \tau \notin \cT$, such that $p_{\pi}(\tau) > 0$, then $\pi$ is not an optimal policy.
\label{lemma:op}
\end{lemma}
\begin{proof}
We prove this by contradiction. We assume $\pi$ is an optimal policy. If $\exists \tau' \notin \cT$, such that 1) $p_{\pi}(\tau') \neq 0$, or equivalently: $p_{\pi}(\tau') > 0$ since $p_{\pi}(\tau')\in[1, 0]$. 
and 2) $\tau' \notin \cT$. We can find a better policy $\pi'$ by satisfying the following three conditions: 
\begin{align*}
	& p_{\pi'}(\tau')  =  0  \text{ and }\\
	& p_{\pi'}(\tau_1) = p_{\pi}(\tau_1) + p_{\pi}(\tau'),  \tau_1 \in \cT \text{ and }\\
	& p_{\pi'}(\tau) =  p_{\pi}(\tau), \forall \tau \notin \{\tau', \tau_1\}
\end{align*}
Since $p_{\pi'}(\tau) \geq 0, \forall \tau $ and $\sum_{\tau}p_{\pi'}(\tau) = 1$, therefore $p_{\pi'}$
constructs a valid probability distribution.
Then the expected long-term performance of $\pi'$ is greater than that of $\pi$:
\begin{align*}
 &	\sum_{\tau}\nolimits  p_{\pi'}(\tau)w(\tau) - \sum_{\tau}\nolimits  p_{\pi}(\tau)w(\tau) \\
=& \sum_{\tau \notin \{\tau', \tau_1\}}\nolimits  p_{\pi'}(\tau)w(\tau) + p_{\pi'}(\tau_1)w(\tau_1)  + p_{\pi'}(\tau')w(\tau')\\
& \qquad - \left( \sum_{\tau \notin \{\tau', \tau_1\}}\nolimits  p_{\pi}(\tau)w(\tau) + p_{\pi}(\tau_1)w(\tau_1)  + p_{\pi}(\tau')w(\tau')\right)\\
=&  p_{\pi'}(\tau_1)w(\tau_1)  + p_{\pi'}(\tau')w(\tau') - (p_{\pi}(\tau_1)w(\tau_1)  + p_{\pi}(\tau')w(\tau'))\\
& \because \tau' \notin \cT, \therefore w(\tau') = 0 \text{ and } \tau_1 \in \cT, \therefore w(\tau) = 1 \\
=& p_{\pi'}(\tau_1) - p_{\pi}(\tau_1) \\
=& p_{\pi}(\tau_1) + p_{\pi}(\tau') - p_{\pi}(\tau_1) =  p_{\pi}(\tau') > 0.
\end{align*}
Essentially, we can find a policy $\pi'$ that has higher probability on the optimal trajectory $\tau_1$
and zero probability on $\tau'$. This indicates that it is a better policy than $\pi$. Therefore, $\pi$ is not an optimal policy and it contradicts our assumption, which proves that such $\tau'$ does not exist. Therefore, $ \forall \tau \notin \cT$, we have $p_{\pi}(\tau) = 0 $. 
\end{proof}

\begin{lemma}[Policy Performance]
If the policy takes the form as in \eq{\ref{eq:top1:p}} or \eq{\ref{eq:rpg:policy}}, then we have 
$\forall \tau$, $p_\theta(\tau) > 0$. This means for all possible trajectories allowed by the environment, the policy takes the form
of either ranking policy or softmax will generate this trajectory with probability $p_\theta(\tau) > 0$.
Note that because of this property, $\pi_\theta$ is not an optimal policy according to Lemma~\ref{lemma:op}, though it can be arbitrarily close to an optimal policy.
\label{lemma:pp}
\end{lemma}
\begin{proof}
\begin{align*}
& \text{The trajectory probability is defined as: } p(\tau) = p(s_1)\Pi_{t=1}^{T}(\pi_{\theta}(a_t|s_t)p(s_{t+1}|s_{t}, a_{t})) \\
& \text{Then we have: }\nonumber \\
& \textrm{The policy takes the form as in \eq{\ref{eq:top1:p}} or \eq{\ref{eq:rpg:policy}} } \Rightarrow \pi_{\theta}(a_t|s_t) > 0. \\
&  p(s_1)>0, p(s_{t+1}|s_{t}, a_{t}) > 0. \Rightarrow p_\theta(\tau) = 0. \\
& p(s_{t+1}|s_{t}, a_{t}) = 0 \text{ or }  p(s_1) = 0, \Rightarrow p_\theta(\tau) = 0, \textrm{ which means } \tau \text{ is not a possible trajectory. }\nonumber \\
& \text{In summary, for all possible trajectories, } p_\theta(\tau) > 0. 
\end{align*}
This thus completes the proof. 
\end{proof}

\section{Proof of Corollary~\ref{cor:rppg}}
\label{sbsec:corrppg}
\begin{corollary}[Ranking performance policy gradient]
The lower bound of expected long-term performance by 
ranking policy can be approximately optimized by the following loss:
\begin{align}
   \min_{\theta} \sum_{s, a_i}\nolimits p_{\pi_*}(s, a_i) L(s_{i}, a_{i}) \label{eq:hingeloss2}
\end{align}
where the pair-wise loss $L(s_{i}, a_{i})$ is defined as:
\begin{align*}
\cL(s, a_{i}) = \sum_{j=1,j\neq i}^{|A|}\nolimits  \max (0 , 1 + \lambda(s, a_j) - \lambda(s,  a_{i}))
\end{align*}
\end{corollary}
\begin{proof}
In RPG, the policy $\pi_{\theta}(a|s)$ is defined as in \eq{\ref{eq:rpg:policy}}. 
We then replace the action probability distribution in \eq{\ref{eq:ltpt:sl}} with the 
RPG policy. 
\begin{align}
 & \because  \pi(a=a_i|s) = \Pi_{j=1, j\neq i}^{m} p_{ij} \\
& \text{Because RPG is fitting a deterministic optimal policy, } \nonumber\\
& \text{we denote the optimal action given sate } s  \text{ as }  a_i, \text{ then we have } \nonumber \\
& \max_{\theta}\sum_{s,a_i}p_{\pi_*}(s, a_i) \log\pi(a_i|s)\\
=& \max_{\theta}\sum_{s,a_i}p_{\pi_*}(s, a_i) \log(\Pi_{j\neq i, j=1}^m p_{ij} ) \\
=& \max_{\theta} \sum_{s, a_i} p_{\pi_*}(s, a_i) \log \Pi_{j\neq i, j=1}^m \frac{1}{1 + e^{\lambda_{ji}}}\\
=& \min_{\theta} \sum_{s, a_i}p_{\pi_*}(s, a_i)\sum_{j\neq i, j=1}^m\log (1 + e^{\lambda_{ji}}) \text{ first order Taylor expansion}\\
\approx & \min_{\theta} \sum_{s, a_i}p_{\pi_*}(s, a_i)\sum_{j\neq i, j=1}^m \lambda_{ji} \hspace{2em} \text{ s.t. }  |\lambda_{ij}| =c < 1, \forall i, j, s\\
= & \min_{\theta} \sum_{s, a_i}p_{\pi_*}(s, a_i)\sum_{j\neq i, j=1}^m (\lambda_j-\lambda_i) \hspace{2em}  \text{ s.t. }   |\lambda_i - \lambda_{j}| = c < 1, \forall i, j, s  \label{eq:hingebefore}\\
\Rightarrow & \min_{\theta} \sum_{s, a_i}p_{\pi_*}(s, a_i) L(s_{i}, a_{i}) \label{eq:hingeafter}\\
& \text{where the pairwise loss $L(s, a_{i})$ is defined as:} \nonumber \\
&\cL(s, a_{i}) = \sum_{j=1,j\neq i}^{|A|} \max (0 , \text{margin}  + \lambda(s, a_j) - \lambda(s,  a_{i})), \\
& \text{where the margin in \eq{\ref{eq:hingeafter}} is a small positive constant. }
\end{align}
From~\eq{\ref{eq:hingebefore}} to~\eq{\ref{eq:hingeafter}}, we consider learning a deterministic optimal policy $a_i = \pi^*(s)$, where
we use index $i$ to denote the optimal action at each state.  The optimal $\lambda$-values minimizing \eq{\ref{eq:hingebefore}} (denoted by $\lambda^1$) need
to satisfy $\lambda^1_i = \lambda^1_j + c, \forall j \neq i, s$. The optiaml $\lambda$-values minimizing \eq{\ref{eq:hingeafter}} (denoted by $\lambda^2$) need to satisfy $\lambda^2_i = \max_{j\neq i} \lambda^2_j + \text{margin}, \forall j \neq i, s$. In both cases, the optimal policies from solving \eq{\ref{eq:hingebefore}} and \eq{\ref{eq:hingebefore}} are the same: $\pi(s) = \argmax_k \lambda^1_k = \argmax_k \lambda^2_k = a_i$. Therefore, we use~\eq{\ref{eq:hingeafter}} as a surrogate optimization problem of~\eq{\ref{eq:hingebefore}}.
\end{proof}


\section{Policy gradient variance reduction}
\label{sec:app:pgvr}

\begin{corollary}[Variance reduction]
Given a stationary policy, 
the upper bound of the variance of each dimension of policy gradient is
$\cO(T^2C^2R^2_{\max})$. The upper bound of gradient variance of maximizing the
lower bound of long-term performance  \eq{\ref{eq:ltpt:sl}} is $\cO(C^2)$, where
$C$ is the maximum norm of log gradient based on Assumption~\ref{ass:gmn}. The supervised learning has reduced the upper bound of gradient variance by an order of $\cO(T^2 R_{\max}^2)$ as compared to the regular policy gradient, 
considering $R_{\max}\geq 1, T\geq 1$,  which is a very common situation in practice.
\end{corollary}
\begin{proof}
	The regular policy gradient of policy $\pi_\theta$ is given as~\citep{williams1992simple}:
\begin{align*}
  \sum_{\tau}\nolimits p_{\theta}(\tau)[\sum_{t=1}^{T}\grad_{\theta} \log(\pi_{\theta}(a(\tau, t)|s(\tau, t))) r(\tau)]
\end{align*}
The regular policy gradient variance of the $i$-th dimension is denoted as follows:
\begin{align*}
  Var\left(\sum_{t=1}^{T}\grad_{\theta} \log(\pi_{\theta}(a(\tau, t)|s(\tau, t))_i) r(\tau)\right)
\end{align*}
We denote $x_i(\tau) = \sum_{t=1}^{T}\grad_{\theta} \log(\pi_{\theta}(a(\tau, t)|s(\tau, t))_i) r(\tau)$ for 
convenience. Therefore, $x_i$ is a random variable. 
Then apply $var(x) = \vE_{p_\theta(\tau)}[x^2]-  \vE_{p_\theta(\tau)}[x]^2$, we have:
\begin{align*}
	& Var\left(\sum_{t=1}^{T}\nolimits \grad_{\theta} \log(\pi_{\theta}(a(\tau, t)|s(\tau, t))_i) r(\tau)\right)\\
	=& Var\left(x_i(\tau)\right)\\
	=& \sum_{\tau}\nolimits p_{\theta} (\tau) x_i(\tau)^2 - [\sum_\tau\nolimits  p_{\theta}(\tau) x_i(\tau)]^2\\
	\leq& \sum_{\tau}\nolimits p_{\theta}(\tau) x_i(\tau)^2\\
	=& \sum_{\tau} \nolimits p_{\theta}(\tau) [\sum_{t=1}^{T}\nolimits \grad_{\theta} \log(\pi_{\theta}(a(\tau, t)|s(\tau, t))_i) r(\tau)]^2 \qquad \\ 
	\leq&\sum_{\tau}\nolimits p_{\theta}(\tau) [\sum_{t=1}^{T}\nolimits \grad_{\theta} \log(\pi_{\theta}(a(\tau, t)|s(\tau, t))_i)]^2R_{max}^2 \\
	=&R_{max}^2 \sum_{\tau}\nolimits p_{\theta}(\tau) [\sum_{t=1}^{T}\nolimits \sum_{k=1}^{T}\nolimits \grad_{\theta} \log(\pi_{\theta}(a(\tau, t)|s(\tau, t))_i)\grad_\theta\log(\pi_{\theta}(a(\tau, k)|s(\tau, k)_i)] \text{ (Assumption~\ref{ass:gmn})}\\
	\leq & R_{max}^2 \sum_{\tau}\nolimits p_{\theta} (\tau) [\sum_{t=1}^{T}\sum_{k=1}^{T}C^2]\\
	=&R_{max}^2 \sum_{\tau}\nolimits p_{\theta}(\tau) T^2C^2\\
	=& T^2C^2R_{max}^2 
\end{align*}
The policy gradient of long-term performance (Def~\ref{def:ltp})
\begin{align*}
\sum_{s,a} \nolimits p_{\pi_*}(s, a) \grad_{\theta} \log\pi_{\theta}(a|s) 	
\end{align*}
The policy gradient variance of the $i$-th dimension is denoted as 
\begin{align*}
	var(\grad_{\theta} \log\pi_{\theta}(a|s)_i) 
\end{align*}
Then the upper bound is given by 
\begin{align*}
& var(\grad_{\theta} \log\pi_{\theta}(a|s)_i) \\
&= \sum_{s,a}\nolimits  p_{\pi_*}(s, a) [\grad_{\theta} \log\pi_{\theta}(a|s)_i]^2	- [\sum_{s,a}\nolimits  p_{\pi_*}(s, a) \grad_{\theta} \log\pi_{\theta}(a|s)_i]^2\\
&\leq \sum_{s,a} \nolimits p_{\pi_*}(s, a) [\grad_{\theta} \log\pi_{\theta}(a|s)_i]^2	\qquad\text{(Assumption~\ref{ass:gmn})}\\
&\leq \sum_{s,a} \nolimits p_{\pi_*}(s, a) C^2\\
&= C^2
\end{align*}

This thus completes the proof. 
\end{proof}

\section{Discussions of Assumption~\ref{ass:uot}}
\label{sec:lemma:op}

In this section, we show that UNOP exists in a range of MDPs. Notice that the lemma~\ref{lemma:uot} shows the sufficient conditions of satisfying Asumption~\ref{ass:uot} rather than necessary conditions.
\begin{lemma}
  For MDPs defined in Section~\ref{sec:notations_and_problem_setting}
   satisfying the following conditions:
   \begin{itemize}
     \item Each initial state leads to one optimal trajectory. This also indicates
     $|\cS_1| = |\cT|$, where $\cT$ denotes the set of optimal trajectories in this lemma, $\cS_1$ denotes the set of initial states. 
     \item Deterministic transitions, i.e., $p(s'|s, a) \in \{0, 1\}$. 
     \item Uniform initial state distribution, i.e., $p(s_1) = \frac{1}{|\cT|}, \forall s_1 \in \cS_1$. 
   \end{itemize}
 Then we have: 
  $\exists \pi_* $, where \ s.t. $p_{\pi_*}(\tau) = \frac{1}{|\cT|}, \forall \tau \in \cT$. 
  It means that a deterministic uniformly optimal policy always exists for this MDP. 
  \label{lemma:uot}
\end{lemma}
\begin{proof}
  We can prove this by construction. The following analysis applies for any $\tau \in \cT$.
  \begin{align}
    &p_{\pi_*}(\tau) = \frac{1}{|\cT|}\nonumber\\
    \Longleftrightarrow & \log p_{\pi_*}(\tau) = -\log|\cT| \nonumber\\
   \Longleftrightarrow & \log p(s_1) + \sum_{t=1}^{T}\nolimits \log p(s_{t+1}|s_{t}, a_{t})+ \sum_{t=1}^{T}\nolimits \log\pi_{*}(a_t|s_t)  = -\log|\cT|\nonumber\\
   \Longleftrightarrow & \sum_{t=1}^{T}\nolimits \log\pi_{*}(a_t|s_t)  = - \log p(s_1) - \sum_{t=1}^{T}\nolimits \log p(s_{t+1}|s_{t}, a_{t}) -\log|\cT|\nonumber\\
   & \text{where we use } a_t, s_t \text{ as abbreviations of } a(\tau,t), s(\tau,t). \nonumber \\
   & \text{We denote } D(\tau) = -\log p(s_1) - \sum_{t=1}^{T}\nolimits \log p(s_{t+1}|s_{t}, a_{t}) > 0 \nonumber \\
   \Longleftrightarrow  &\sum_{t=1}^{T}\nolimits \log\pi_{*}(a_t|s_t)  =  D(\tau) -\log|\cT|\nonumber\\
 & \therefore \text{we can obtain a uniformly optimal policy by solving the nonlinear programming: } \nonumber \\
  & \sum_{t=1}^{T} \nolimits \log\pi_{*}(a(\tau, t)|s(\tau, t)) =  D(\tau) -\log|\cT|  \ \forall \tau \in \cT \label{eq:luotcond1}\\
 & \log\pi_{*}(a(\tau, t)|s(\tau, t)) = 0, \ \forall \tau \in \cT, t=1,...,T \label{eq:luotcond2}\\
 & \sum_{i=1}^m \nolimits \pi_{*}(a_i|s(\tau, t)) = 1,  \ \forall \tau \in \cT, t=1,...,T \label{eq:luotcond3}
  \end{align}
Use the condition $p(s_1) = \frac{1}{|\cT|}$, then we have:
  \begin{align}
   & \because \sum_{t=1}^{T}\nolimits  \log\pi_{*}(a(\tau, t)|s(\tau, t)) = \sum_{t=1}^{T}\nolimits  \log 1  = 0 \ (\text{LHS of }~\eq{\ref{eq:luotcond1}})\nonumber\\
& \because - \log p(s_1) - \sum_{t=1}^{T}\nolimits \log p(s_{t+1}|s_{t}, a_{t}) -\log|\cT| = \log|\cT| - 0 -\log|\cT| =0  \ (\text{RHS of }~\eq{\ref{eq:luotcond1}})\nonumber \\
 & \therefore  D(\tau) -\log|\cT|  = \sum_{t=1}^{T} \nolimits \log\pi_{*}(a(\tau, t)|s(\tau, t)) , \ \forall \tau \in \cT.\nonumber
  \end{align}
  Also the deterministic optimal policy satisfies the conditions in \eq{\ref{eq:luotcond2}~\ref{eq:luotcond3}}. 
  Therefore, the deterministic optimal policy is a uniformly optimal policy. This lemma describes one type of MDP in which UOP exists. From the above reasoning, we can see that 
  as long as the system of non-linear equations~\eq{\ref{eq:luotcond1}~\ref{eq:luotcond2}~\ref{eq:luotcond3}}
  has a solution, the uniformly (near)-optimal policy exists. 
\end{proof}

\begin{lemma}[Hit optimal trajectory]     
The probability that a specific optimal trajectory was not encountered given
an arbitrary softmax policy $\pi_\theta$ is exponentially decreasing with
respect to the number of training episodes. No matter a MDP has
deterministic or probabilistic dynamics.
\label{lemma:hot} 
\end{lemma}
\begin{proof}
	Given a specific optimal trajectory $\tau = \{s(\tau, t), a(\tau, t)\}_{t=1}^T$, and an arbitrary 
	stationary policy $\pi_\theta$, the probability that has never encountered at the $n$-th episode is
	$[1 - p_{\theta}(\tau)]^n = \xi^n$, based on lemma~\ref{lemma:pp}, we have 
	$p_{\theta}(\tau) > 0$, therefore we have $\xi \in [0, 1)$. 
\end{proof}



\section{Discussions of Assumption~\ref{ass:rpgind}}
\label{sec:asm1}

\begin{figure}
\center
	\begin{tabular}{cc}	
	\includegraphics[width=0.4\textwidth]{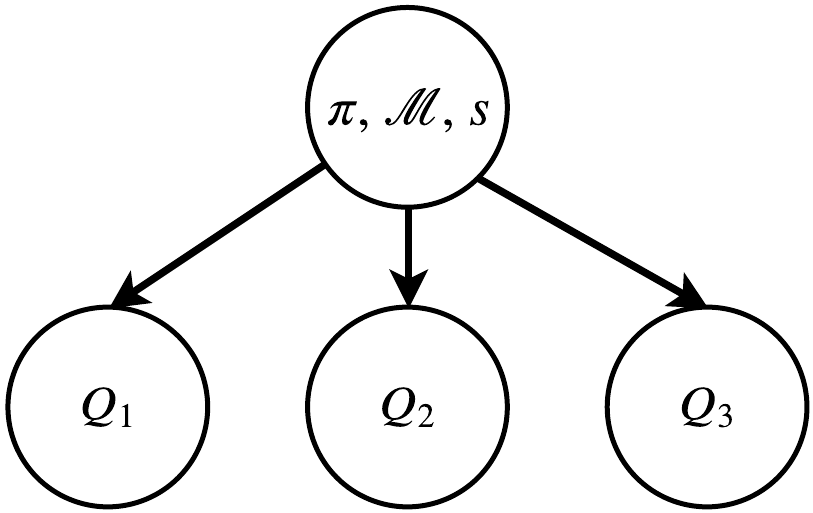} & 
	\includegraphics[width=0.4\textwidth]{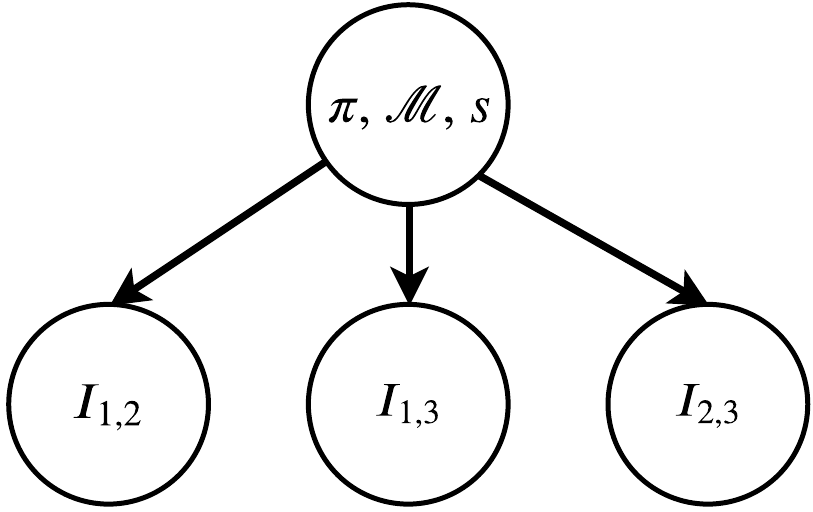}\\
	(a) & (b)
	\end{tabular}
	\caption{
The directed graph that describes the conditional independence of
pairwise relationship of actions, where $Q_1$ denotes the return of taking 
action $a_1$ at state $s$, following policy $\pi$ in $\cM$, i.e., $Q^\pi_\cM(s, a_1)$.
$I_{1,2}$ is a random variable that denotes the pairwise relationship of $Q_1$ and $Q_2$, i.e.,
$I_{1,2} = 1, \ \mathrm{ i.i.f. }\  Q_1 \geq Q_2, \mathrm{\ o.w.\ } I_{1,2} = 0$. }
\label{fig:asm1}
\vspace{-1em}
\end{figure}

Intuitively, given a state and a stationary policy $\pi$, the relative relationships
among actions can be independent, considering a fixed MDP $\cM$. 
The relative relationship among actions
is the relative relationship of actions' return. Starting from the same state, 
following a stationary policy, the actions' return is determined by
MDP properties such as environment dynamics, reward function, etc. 

More concretely, we consider a MDP with three actions $(a_1, a_2, a_3)$
for each state. The action value $Q^{\pi}_\cM$ satisfies the Bellman equation in~\eq{\ref{eq:bellman}}.
Notice that in this subsection, we use $Q^{\pi}_\cM$ to denote the action value that 
estimates the absolute value of return in $\cM$. 
\begin{align}
	Q^{\pi}_\cM(s, a_i) = r(s, a_i) + \max_{a} \vE_{s'\sim p(*|s, a)} Q^{\pi}_\cM(s', a), \forall i = 1,2,3.\label{eq:bellman}
\end{align}
As we can see from~\eq{\ref{eq:bellman}}, $Q^{\pi}_\cM(s, a_i), i=1,2,3$ is only
related to $s, \pi$, and environment dynamics $\vP$. It means if $\pi$, $\cM$
and $s$ are given, the action values of three actions are determined.
Therefore, we can use a
directed graph~\citep{bishop2006pattern} to model the relationship of action values, as shown in
Figure~\ref{fig:asm1} (a). Similarly, if we only consider the ranking of
actions, this ranking is consistent with the relationship of actions' return,  
which is also determined by $s, \pi$, and $\vP$. Therefore, the pairwise relationship among
actions can be described as the directed graph in Figure~\ref{fig:asm1} (b), which establishes
the conditional independence of actions' pairwise relationship. Based on the above
reasoning, we conclude that Assumption~\ref{ass:rpgind} is realistic.






\section{The proof of Theorem~\ref{th:rlgp}}
\label{sec:app:subthrlgp}
\begin{proof}
The proof mainly establishes on the proof for long term performance Theorem~\ref{th:ltpt}
and connects the generalization bound in PAC framework to the lower bound of return.
\begin{align}
 & \log (\frac{1}{|\cT|}\sum_{\tau\in \cT} \nolimits p_{\theta}(\tau)w(\tau) ) \geq  \frac{1}{|\cT|}\sum_{\tau\in \cT}  \log p_{\theta}(\tau)w(\tau) \nonumber\\
 \Leftrightarrow & \sum_{\tau\in \cT}\nolimits  p_{\theta}(\tau)w(\tau)  \geq |\cT| \exp(\frac{1}{|\cT|}\sum_{\tau\in \cT}  \log p_{\theta}(\tau)w(\tau)) \nonumber\\
 &  \text{  denote } F = \sum_{\tau}\nolimits  p_{\theta}(\tau)w(\tau) = \sum_{\tau\in \cT} \nolimits p_{\theta}(\tau)w(\tau) \nonumber \\
 \Leftrightarrow F & \geq |\cT| \exp(\frac{1}{|\cT|}\sum_{\tau\in \cT} \nolimits  \log p_{\theta}(\tau)w(\tau)) \nonumber\\
  &= |\cT| \exp\left( \frac{1}{|\cT|}\sum_{\tau\in \cT} \nolimits \left(\log p(s_1) + \sum_{t=1}^{T}\nolimits \log p(s_{t+1}|s_{t}, a_{t})+ \sum_{t=1}^{T}\nolimits \log\pi_{\theta}(a_t|s_t) + \log w(\tau) \right) \right) \nonumber\\
  & \because w(\tau)= 1, \ \forall \tau \in \cT, s_t = s(\tau, t), a_t=a(\tau, t), t =1,\dots, T \nonumber\\
  &= |\cT| \exp\left( \frac{1}{|\cT|}\sum_{\tau\in \cT} \nolimits \left(\log p(s_1) + \sum_{t=1}^{T}\nolimits \log p(s_{t+1}|s_{t}, a_{t})+ \sum_{t=1}^{T}\log\pi_{\theta}(a_t|s_t) \right) \right) \nonumber\\
  &= |\cT| \exp\left( \frac{1}{|\cT|}\sum_{\tau\in \cT} (\log p(s_1) + \sum_{t=1}^{T}\nolimits \log p(s_{t+1}|s_{t}, a_{t}))\right)\exp\left( \frac{1}{|\cT|}\sum_{\tau\in \cT}\nolimits (\sum_{t=1}^{T}\nolimits \log\pi_{\theta}(a_t|s_t)) \right) \nonumber\\
  & \text{ Denote the dynamics of a trajectory as } p_d(\tau) = p(s_1)\Pi_{t=1}^{T}p(s_{t+1}|s_{t}, a_{t})  \nonumber\\
  & \text{ Notice that } p_d(\tau) \text{ is environment dynamics, which is fixed given a specific MDP.} \nonumber\\
  \Leftrightarrow  F & \geq |\cT| \exp\left( \frac{1}{|\cT|}\sum_{\tau\in \cT}\nolimits \log p_d(\tau)\right)\exp\left( \frac{1}{|\cT|}\sum_{\tau\in \cT}\nolimits (\sum_{t=1}^{T}\nolimits \log\pi_{\theta}(a_t|s_t)) \right) \nonumber\\
  & = |\cT|\left(\Pi_{\tau\in \cT} p_d(\tau)\right)^{\frac{1}{|\cT|}}\exp\left( \frac{1}{|\cT|T}\sum_{\tau\in \cT}\nolimits (\sum_{t=1}^{T}\nolimits \log\pi_{\theta}(a_t|s_t)) T \right) \nonumber\\
  & \text{Use the same reasoning from \eq{\ref{eq:ltpt:sumtraj}} to \eq{\ref{eq:ltpt:slapp}}. }\nonumber\\
  & = |\cT|\left(\Pi_{\tau\in \cT} p_d(\tau)\right)^{\frac{1}{|\cT|}}\exp\left( T\sum_{s,a}\nolimits p_{\pi_*}(s, a) \log\pi_{\theta}(a|s) \right) \nonumber\\
  & = |\cT|\left(\Pi_{\tau\in \cT} p_d(\tau)\right)^{\frac{1}{|\cT|}}\exp(TL) \nonumber\\
  & \text{Apply Corollary~\ref{cor:rppg} and denote } L = \sum_{s,a} p_{\pi_*}(s, a) \log\pi_{\theta}(a|s). \nonumber\\
  & L \text{ is the only term that is related to the policy parameter } \theta \nonumber \\
  L &= \sum_{s,a \in U_w}\nolimits  p_{\pi_*}(s, a) \log\pi_{\theta}(a|s) + \sum_{s,a \notin U_w } \nolimits p_{\pi_*}(s, a) \log\pi_{\theta}(a|s)\nonumber\\
  & \because \text{ with current policy classifier } \theta, \forall s,a \notin U_w, \pi_{\theta}(a|s) = 1\nonumber\\
  & = \sum_{s,a \in U_w}\nolimits  p_{\pi_*}(s, a) \log\pi_{\theta}(a|s)\nonumber\\
  & \text{ If we use RPG as our policy parameterization, then with } \eq{\ref{eq:rpg:policy}}\nonumber\\
  & = \sum_{s,a\in U_w} \nolimits p_{\pi_*}(s, a)\log(\Pi_{j\neq i, j=1}^m p_{ij} ) \nonumber\\
  & = \sum_{s, a_i \in U_w }\nolimits p_{\pi_*}(s, a_i)\sum_{j\neq i, j=1}^m\nolimits \log \frac{1}{1 + e^{Q_{ji}} }\nonumber\\
  & \text{By Condition~\ref{ass:av}, which can be easily satisfied in practice. Then we have: } Q_{ij} < 2c_q \leq 1 \nonumber \\
  &\text{Given } h = \pi_\theta, \text{misclassified state action pairs set } U_w = \{s, a| h(s)\neq a, (s,a)\sim p_*(s,a)\} \nonumber\\
  &\text{Apply Lemma~\ref{lemma:samplecomplexiy}, the misclassified rate is at most } \eta. \nonumber\\
  & \geq \sum_{s, a_i \in U_w}\nolimits p_{\pi_*}(s, a_i)(m-1)\log (\frac{1}{1 + e}) \nonumber\\
  & \geq -\sum_{s, a_i \in U_w}\nolimits p_{\pi_*}(s, a_i)(m-1)\log (1 + e) \nonumber\\
  & \geq -\eta(m-1)\log (1 + e) \nonumber\\
  & = \eta(1-m)\log (1 + e) \nonumber \\
  F &\geq |\cT|\left(\Pi_{\tau\in \cT} p_d(\tau)\right)^{\frac{1}{|\cT|}}\exp(TL) \nonumber \\
    &\geq |\cT|\left(\Pi_{\tau\in \cT} p_d(\tau)\right)^{\frac{1}{|\cT|}}\exp(\eta(1-m)T\log (1 + e)) \nonumber\\
    &\geq |\cT|\left(\Pi_{\tau\in \cT} p_d(\tau)\right)^{\frac{1}{|\cT|}}(1 + e)^{\eta(1-m)T} \nonumber\\
    & = D(1 + e)^{\eta(1-m)T} \nonumber
\end{align}

From generalization performance to sample complexity: 
\begin{align*}
  & \mathrm{Set }\ 1 - \epsilon = D(1 + e)^{\eta(1-m)T} , \text{where } D = |\cT|\left(\Pi_{\tau\in \cT} p_d(\tau)\right)^{\frac{1}{|\cT|}} \nonumber \\
  & \eta = \frac{\log_{1+e}\frac{D}{1- \epsilon}}{(m-1)T} \nonumber\\
  & \text{With realizable assumption~\ref{ass:realiz}, } \epsilon_{\min} = 0 \\
  & \gamma = \frac{\eta - \epsilon_{\min}}{2} = \frac{\eta}{2} \\
  n &\geq \frac{1}{2\gamma^2} \log\frac{2|\cH|}{\delta} \\
    & =  \frac{2(m-1)^2T^2}{\left(\log_{1+e}\frac{D}{1- \epsilon}\right)^2} \log\frac{2|\cH|}{\delta}\\
\end{align*}

Bridge the long-term reward and long-term performance:
\begin{align}
 & \sum_{\tau} \nolimits p_{\theta}(\tau)r(\tau) ~\textrm{In Section~\ref{sec:samcomplex},  } r(\tau) \in (0, 1], \forall \tau.\nonumber \\
\geq&  \sum_{\tau}\nolimits  p_{\theta}(\tau)w(\tau)  \textrm{ Since we focus on UOP Def~\ref{def:uop}, c = 1 in TSR Def~\ref{def:PI}} \nonumber \\
= & \sum_{\tau\in \cT}\nolimits  p_{\theta}(\tau)w(\tau) \nonumber \\
\geq & 1 - \epsilon \nonumber
\end{align}
This thus concludes the proof. 
\end{proof}

\begin{assumption}[Realizable]
We assume there exists a hypothesis $h_* \in \cH$ that obtains zero expected risk, i.e.
$\exists h_* \in \cH\ \Rightarrow  \ \sum_{s, a}p_{\pi_*}(s,a) \textbf{1}\{ h_*(s) \neq a\} = 0$.
\label{ass:realiz} 
\end{assumption}
The Assumption~\ref{ass:realiz} is not necessary for the proof of Theorem~\ref{th:rlgp}. 
For the proof of Corollary~\ref{coro:rlgp}, we introduce this assumption to 
achieve more concise conclusion. In finite MDP, the realizable assumption 
can be satisfied if the policy is parameterized by multi-layer neural newtwork, due to
its perfect finite sample expressivity~\citep{zhang2016understanding}. 
It is also advocated in our empirical studies since the neural network achieved 
optimal performance in \textsc{Pong}.



\section{The proof of Lemma~\ref{lemma:osr}}
\label{sec:app:sublemmaosr}
\begin{proof}
Let $e_{=i}$ denotes the event $n = i|k$, i.e. 
the number of different optimal trajectories in first $k$ episodes is
equal to $i$. Similarly, $e_{\geq i}$ denotes the event  $n \geq i|k$. 
Since the events $e_{=i}$ and $e_{=j}$ are mutually exclusive when 
$i \neq j$. Therefore, $p(e_{\geq i}) = p(e_{=i}, e_{=i+1}, ..., e_{=|\cT|}) 
= \sum_{j=i}^{|\cT|} p(e_{=j})$. Further more, we know that $\sum_{i=0}^{\cT}p(e_{=i})= 1$ since $\{e_{=i}, i = 0,...,|\cT|\}$ constructs an universal set. 
For example, $p(e_{\geq 1}) = p_{\pi_r, \cM}(n \geq 1|k) = 1 - p_{\pi_r, \cM}(n =0|k) = 1 - (\frac{N - |\cT|}{N})^k $. 
\begin{align}
	p_{\pi_r, \cM}(n \geq i|k) 
	 &= 1  - \sum_{i'=0}^{i-1}\nolimits p_{\pi, \cM}(n = i'|k) \nonumber \\
  & = 1- \sum_{i'=0}^{i-1} \nolimits C_{|\cT|}^{i'} \frac{\sum_{j=0}^{i'}\nolimits (-1)^j C_{i'}^j( N - |\cT| + i' -j)^k }{N^k} \label{eq:in-ex}
\end{align}
In \eq{\ref{eq:in-ex}}, we use the inclusion-exclusion principle~\citep{kahn1996inclusion} to have the following equality. 
\begin{align*}
	p_{\pi_r, \cM}(n = i'|k) &=  C_{|\cT|}^{i'} p(e_{\tau_1,\tau_2,...,\tau_{i'}} )\\
	&=  C_{|\cT|}^{i'} \frac{\sum_{j=0}^{i'}(-1)^j C_{i'}^j(N - |\cT| + i' -j)^k }{N^k} 
\end{align*}
$e_{\tau_1,\tau_2,...,\tau_{i'}} $ denotes the event: 
in first $k$ episodes, a certain set of $i'$ 
optimal trajectories $\tau_1,\tau_2,...,\tau_{i'}, i'\leq |\cT|$ is sampled.
\end{proof}


\section{The proof of Corollary~\ref{coro:esrlgp}}
\label{sec:app:eerlgp}

\begin{proof}
The Corollary~\ref{coro:esrlgp} is a direct application of Lemma~\ref{lemma:osr} and 
Theorem~\ref{th:rlgp}. First, we reformat Theorem~\ref{th:rlgp} as follows:
$$ p(A|B)\geq 1 - \delta$$
where event $A$ denotes $\sum_{\tau} p_{\theta}(\tau)r(\tau) \geq  D(1 + e)^{\eta(1-m)T}$, 
event $B$ denotes the number of state-action pairs $n'$ from UOP (Def~\ref{def:uop}) satisfying $n' \geq n$, given fixed $\delta$. With Lemma~\ref{lemma:osr}, we have
$p(B) \geq p_{\pi_r, \cM}(n' \geq n |k)$. Then, $P(A) = P(A|B)P(B)\geq (1 - \delta)p_{\pi_r, \cM}(n' \geq n |k)$. 
\begin{align}
	&\textrm{Set } (1 - \delta)p_{\pi_r, \cM}(n' \geq n |k) = 1 - \delta' \nonumber \\ 
    &\textrm{we have } P(A) \geq  1- \delta' \nonumber \\
    & \delta =  1- \frac{1 - \delta'}{p_{\pi_r, \cM}(n' \geq n |k)} \nonumber \\
     \eta  & = 2\sqrt{\frac{1}{2n}\log \frac{2|\cH|}{\delta}}  \nonumber \\
           & = 2\sqrt{\frac{1}{2n}\log \frac{2|\cH|p_{\pi_r, \cM}(n' \geq n |k)}{p_{\pi_r, \cM}(n' \geq n |k) - 1 + \delta'}} \nonumber
\end{align}

\end{proof}


\subsection{Hyperparameters} 
\label{sec:hyperparameters}
We present the training details of ranking policy gradient in Table~\ref{tb:hyperpara}. 
The network architecture is the same as the convolution neural network used in DQN~\cite{mnih2015human}.
We update the RPG network every four timesteps with a minibatch of size 32.
The replay ratio is equal to eight for all baselines and RPG (except for ACER we use the default setting in openai baselines~\cite{baselines} for better performance). 

\begin{table}[h!]
\caption{Hyperparameters of RPG network}
\centering

\begin{tabular}{ll}\hline
Hyperparameters & Value \\ \hline
Architecture   &  Conv(32-8$\times$8-4) \\
& -Conv(64-4$\times$4-2) \\
& -Conv(64-3$\times$3-2) \\
& -FC(512) \\
Learning rate  &  0.0000625 \\
Batch size     &   32\\
Replay buffer size & 1000000\\
Update period & 4  \\
Margin in \eq{\ref{eq:rankingv2}} & 1 \\\hline
\end{tabular}
\label{tb:hyperpara}
\end{table}

\end{appendices}
\end{document}